\documentclass[twoside,dvipsnames]{article}

\usepackage[accepted]{aistats2024_prelim}
\usepackage{hyperref}
\usepackage{natbib}

\usepackage{graphicx}
\usepackage{tikz}
\usetikzlibrary{calc}
\usetikzlibrary{decorations.pathmorphing}

\usetikzlibrary{shapes}
\usetikzlibrary{arrows.meta,backgrounds,automata}
\usetikzlibrary{positioning,shapes,matrix,calc}
\tikzstyle{vertex}=[circle, draw, fill=gray!80!white,thick,scale=1.2]
\tikzstyle{edge}=[draw=black, thick,-]

\usepackage[most]{tcolorbox}

\tcbset{mystyle/.style={
		enhanced, 
		coltitle = black, 
		colframe = black!15,
		attach boxed title to top left={yshift=-.25mm-\tcboxedtitleheight/2, xshift=2ex},
		boxed title style={enhanced, 
			frame code={\draw[black!15, line width=.5mm, fill=white, 
				rounded corners=1.75ex] 
				(frame.south west) rectangle (frame.north east);}, 
			interior empty}
}}

\newtcbox{\mybox}[2][]{mystyle, #1, title=#2}
\newtcbox{\mytikzbox}[2][]{mystyle, tikz upper, #1, title=#2}

\usepackage{subcaption}
\usepackage{booktabs}
\usepackage{multirow}
\usepackage{url}
\usepackage{paralist}
\usepackage[stable]{footmisc}
\usepackage{adjustbox}
\usepackage{bm}
\usepackage{amsmath}
\usepackage{amssymb}
\usepackage{amsthm}
\usepackage{amsfonts}
\usepackage{thmtools}		
\usepackage{mleftright}
\usepackage{nicefrac}
\usepackage{algorithm}
\usepackage{algorithmic}

\usepackage{todonotes}

\let\originalleft\left
\let\originalright\right
\renewcommand{\left}{\mathopen{}\mathclose\bgroup\originalleft}
\renewcommand{\right}{\aftergroup\egroup\originalright}

\theoremstyle{definition}
\newtheorem{theorem}{Theorem}

\newtheorem{proposition}[theorem]{Proposition}

\newtheorem{lemma}[theorem]{Lemma}

\usepackage{thm-restate}
\usepackage[mathic=true]{mathtools}
\usepackage{fixmath}
\usepackage{siunitx}

\usepackage{pifont}
\newcommand{\cmark}{\ding{51}}
\newcommand{\xmark}{\ding{55}}

\usepackage{enumitem}
\setlist[enumerate]{itemsep=0.2ex, topsep=0.5\topsep}
\setlist[description]{itemsep=0.2ex, topsep=0.5\topsep}
\setlist[itemize]{itemsep=0.2ex, topsep=0.5\topsep}

\makeatletter
\def\thmt@refnamewithcomma #1#2#3,#4,#5\@nil{%
\@xa\def\csname\thmt@envname #1utorefname\endcsname{#3}%
\ifcsname #2refname\endcsname
\csname #2refname\expandafter\endcsname\expandafter{\thmt@envname}{#3}{#4}%
\fi
}
\makeatother

\usepackage[capitalise,noabbrev]{cleveref}   

\usepackage{microtype}
\usepackage{ellipsis}

\newcommand{\mG}{\mathbf{G}}

\newcommand{\cO}{\mathcal{O}}

\newcommand{\Nb}{\mathbb{N}}
\newcommand{\Qb}{\mathbb{Q}}
\newcommand{\Rb}{\mathbb{R}}

\newcommand{\vc}[1]{\mathsf{VC}(#1)}

\newcommand{\UPD}{\mathsf{UPD}}

\newcommand{\new}[1]{\emph{#1}}

\newcommand{\trans}{^T}
\renewcommand{\vc}[1]{\bm{#1}}
\newcommand{\oms}{\{\!\!\{}
\newcommand{\cms}{\}\!\!\}}
\newcommand{\tup}[1]{{(#1)}}

\newcommand*{\tran}{^{\mkern-1.5mu\mathsf{T}}}

\DeclareFontFamily{U}{mathx}{\hyphenchar\font45}
\DeclareFontShape{U}{mathx}{m}{n}{<-> mathx10}{}
\DeclareSymbolFont{mathx}{U}{mathx}{m}{n}
\DeclareMathAccent{\widebar}{0}{mathx}{"73}

\DeclarePairedDelimiterX{\norm}[1]{\lVert}{\rVert}{#1}

\newcommand{\xhdr}[1]{{\noindent\bfseries #1}}

\begin{document}

\twocolumn[

	\aistatstitle{Exploring the Power of Graph Neural Networks in Solving Linear Optimization Problems}

	\aistatsauthor{Chendi Qian \And Didier Chételat \And Christopher Morris}

	\aistatsaddress{RWTH Aachen University\And Polytechnique Montréal \And RWTH Aachen University} ]

\begin{abstract}
	Recently, machine learning, particularly message-passing graph neural networks (MPNNs), has gained traction in enhancing exact optimization algorithms. For example, MPNNs speed up solving mixed-integer optimization problems by imitating computational intensive heuristics like strong branching, which entails solving multiple linear optimization problems (LPs). Despite the empirical success, the reasons behind MPNNs' effectiveness in emulating linear optimization remain largely unclear. Here, we show that MPNNs can simulate standard interior-point methods for LPs, explaining their practical success. Furthermore, we highlight how MPNNs can serve as a lightweight proxy for solving LPs, adapting to a given problem instance distribution. Empirically, we show that MPNNs solve LP relaxations of standard combinatorial optimization problems close to optimality, often surpassing conventional solvers and competing approaches in solving time.
\end{abstract}

\section{\uppercase{Introduction}}

Recently, there has been a surge of interest in training \new{message-passing graph neural networks} (MPNNs) to imitate steps of classical algorithms, such as for shortest-path problems \cite{Cap+2021,velickovic2020neural}. As many of those problems can be formulated as linear optimization problems (LPs), it is natural to ask whether MPNNs could be trained to solve general LPs, at least approximately.

Another recent line of research makes this question particularly intriguing. In integer linear optimization, state-of-the-art solvers all rely on the branch-and-bound algorithm, in which one must repeatedly select variables, subdividing the search space. The best-known heuristic for variable selection is known as ``strong branching,'' which entails solving LPs to score the variables. This heuristic is, unfortunately, too computationally expensive to use in practice. However, in recent years, there has been a collection of works \citep{Gas+2019, Gup+2020, Gup+2022, Nai+2020, Sey+2023} that have proposed to use MPNNs to imitate strong branching, with impressive success. No theoretical explanation has ever been put forward to explain this success. However, perhaps the most straightforward explanation would be that MPNNs implicitly learn to imitate the LP solving underlying strong branching.

\citet{Che+2023} provide a first step towards an explanation. In this work, the authors propose to encode LPs as bipartite graphs in the spirit of \citep{Gas+2019} and show that MPNNs, in principle, can learn to predict the optimal solution of an LP instance up to arbitrary small $\varepsilon$ concerning the supremum norm. They also provide some small-scale experiments suggesting that MPNNs can learn to approximate LP solutions surprisingly well. While a step in the right direction, their theoretical result heavily relies on invoking the universal approximation theorem for multi-layer perceptrons~\cite{Cyb+1992,Les+1993}, and therefore does not explain why modestly-sized MPNNs could be particularly effective LP solvers in practice.

In this paper, we present instead a more specific explanation. We show that several variants of \new{interior-point methods} (IPMs)~\citet{gondzio2012,nocedal2006numerical}, an established class of polynomial-time algorithms for solving LPs, can be interpreted as an MPNN with a specific architecture and choice of parameters that takes as input a graph representation of the LP. Specifically, for two common IPM variants, we show that a sequence of standard MPNN steps can emulate a single iteration of the algorithm on a tripartite (rather than bipartite) graph representation of the LP. This novel theoretical result suggests several conclusions. First, it indicates that MPNNs appear successful at imitating LP solving because they might often be imitating IPMs and that there is a close connection between MPNNs and this specific LP algorithm.
Secondly, although the MPNN's architecture in our theoretical result involves many MPNN layers, MPNNs are a natural choice of machine learning model for LP solving. It is likely possible to approximately solve LPs with fewer layers.

To prove the last hypothesis, we train MPNNs with several layers less than predicted by our theoretical results to imitate the output of a practical IPM algorithm for LP solving, resulting in the \new{IPM-MPNN} architecture. Our empirical results show that IPM-MPNNs can lead to reduced solving times compared to a state-of-the-art LP solver with time constraints and competing neural network-based approaches; see~\cref{fig:overview} for an overview of our approach.

\emph{In summary, our findings significantly contribute to the theoretical framework of data-driven exact optimization using MPNNs, and we also showcase the potential of MPNNs in serving as a light-weight proxy for solving LPs.}

\subsection{Additional related works}
In the following, we discuss relevant related work.

\xhdr{MPNNs} MPNNs~\citep{Gil+2017,Sca+2009} have emerged as a flexible framework for machine learning on graphs and relational data. Notable instances of this architecture include, e.g.,~\cite{Duv+2015,Ham+2017,Vel+2018}, and the spectral approaches proposed in, e.g.,~\cite{Bru+2014,Defferrard2016,Kip+2017}---all of which descend from early work in~\cite{bas+1997,Kir+1995,Mer+2005,mic+2009,mic+2005,Sca+2009,Spe+1997}.

\xhdr{Machine learning for combinatorial optimization} \citet{Ben+2021} discuss and review machine learning approaches to enhance combinatorial optimization (CO). Concrete examples include the imitation of computationally intensive variable selection rules within the branch-and-cut framework~\citep{khalil2016,zarpellon2020}, learning to run (primal) heuristics~\citep{khalil2017a,chmiela2021learning}, learning decompositions of large MILPs for scalable heuristic solving~\citep{song2020}, learning to generate cutting planes~\citep{Dez+2023} or leveraging machine learning to find (primal) solutions to stochastic integer problems quickly~\citep{bengio2020}---see~\citet{kotary2021end} for a high-level overview of recent advances.

\xhdr{MPNNs for CO} Many prominent CO problems involve graph or relational structures, either directly given as input or induced by the variable-constraint interactions. Recent progress in using MPNNs to bridge the gap between machine learning and combinatorial optimization is surveyed in~\cite{Cap+2021}.

Most relevant to the present work, \citet{Gas+2019} proposed to encode the variable-constraint interaction of a mixed-integer linear optimization problem as a bipartite graph and trained MPNNs in a supervised fashion to imitate the costly strong branching heuristic, which entails solving multiple linear optimization problems, within the branch-and-cut framework~\citep{Ach+2005}. Building on that,~\citet{Gup+2020} proposed a hybrid branching model using an MPNN at the initial decision point and a light multi-layer perceptron for subsequent steps, showing improvements on pure CPU machines. Subsequently, \citet{Nai+2020} expanded the MPNN approach to branching by implementing a GPU-friendly parallel linear programming solver using the alternating direction method of multipliers that allows scaling the strong branching expert to substantially larger instances, also combining this innovation with a novel MPNN approach to diving. For all the above works, it remains largely unclear why MPNNs are good at (approximately) predicting strong branching scores. Moreover,~\citet{Kha+2022} used MPNNs to predict the probability of variables being assigned to $0$ or $1$ in near-optimal solutions of binary-integer linear optimization problems.

\citet{Ding2019} used MPNNs on a tripartite graph consisting of variables, constraints, and a single objective node enriched with hand-crafted node features. The target is to predict the 0-1 values of the so-called \emph{stable variables}, i.e., variables whose assignment does not change over a set of pre-computed feasible solutions. \cite{li2022learning} used MPNNs on a bipartite graph together with a pointer network~\citep{Bel+2016} to reorder the variables of a given LP instance, resulting in reduced solving time. \citet{Fan+2023} leveraged MPNNs to find good initial basis solutions for the Simplex algorithm.

Finally~\cite{Wu+2023} expressed an LP as an ordinary differential equations system whose state solution converges to the LP's optimal solution and trained a feed-forward neural network to approximate this state solution. However, their approach hinges on the need to compute the Jacobian matrix, rendering its training phase computationally costly.

\begin{figure*}
	\begin{subfigure}[]{.25\linewidth}
		\vspace{35pt}
		\centering\scalebox{0.65}{\begin{tikzpicture}[scale=0.9]

\tikzstyle{mybox} = [draw=red, fill=blue!20, very thick,
rectangle, rounded corners, inner sep=10pt, inner ysep=20pt]
\tikzstyle{fancytitle} =[fill=red, text=white]

\def\bend{15}
\def\opac{0.2}

\tikzset{line/.style={draw,line width=1.5pt}}
\tikzset{arrow/.style={line,->,>=stealth}}
\tikzset{snake/.style={arrow,line width=1.3pt,decorate,decoration={snake,amplitude=1,segment length=6,post length=7}}}
\tikzset{box/.style={dash pattern=on 5pt off 2pt,inner sep=5pt,rounded corners=3pt}}
\tikzset{node/.style={shape=circle,inner sep=0pt,minimum width=15pt,line width=1pt}}

\tikzset{light/.style={fill=gray}}\tikzset{light/.style={shading=axis,left color=lightgray,right color=black,shading angle=-45}}
\tikzset{light/.style={fill=gray}}

\tikzset{pics/bar/.style args={#1/#2/#3}{code={
\node[inner sep=0pt,minimum height=0.3cm] (#1) at (0,-0.15) {};

\draw[fill=blue,blue] (-0.15,-0.2) rectangle (-0.075,0.25*#2-0.2);
\draw[fill=Orchid,Orchid] (-0.025,-0.2) rectangle (0.05,0.25*#3-0.2);

\begin{scope}[transparency group,opacity=\opac]
\draw[fill=red,light] (-0.15,-0.2) rectangle (-0.075,0.25*#2-0.2);
\draw[fill=Orchid,light] (-0.025,-0.2) rectangle (0.05,0.25*#3-0.2);
\end{scope}

\draw[arrow,line width=1pt] (-0.2,-0.2) -- (0.2,-0.2);
\draw[arrow,line width=1pt] (-0.2,-0.21) -- (-0.2,0.2);
}}}

\node[line,node] (c1) at (0.5, 0) {$c_1$};
\node[line,node] (c2) at (0.5, -1.5) {$c_2$};
\node[line,node] (x1) at (1.5, 0.75) {$x_1$};
\node[line,node] (x2) at (1.5, -0.75) {\textcolor{black}{$x_2$}};
\node[line,node] (x3) at (1.5, -2.25) {$x_3$};

\node[line,node] (o1) at (3, -0.75) {$o$};

\begin{scope}[transparency group,opacity=\opac]

\node[line,node,light] at (c1) {};
\node[line,node,light] at (c2) {};
\node[line,node,light] at (x1) {};
\node[line,node,light] at (x2) {};
\node[line,node,light] at (x3) {};
\node[line,node,light] at (o1) {};

\end{scope}

\path[line,lightgray] (x1) to (c1);
\path[line,lightgray] (c1) to (x2);
\path[line,lightgray] (x2) to (c2);
\path[line,lightgray] (c2) to (x3);

\path[line,lightgray] (o1) to (x1);
\path[line,lightgray] (o1) to (x2);
\path[line,lightgray] (o1) to (x3);
\path[line,lightgray] (o1) to (c1);
\path[line,lightgray] (o1) to (c2);

\end{tikzpicture}
}
		\vspace{23pt}
		\caption{Representing an LP instance with two constraints and three variables as a tripartite graph. \label{fig:lp}}
	\end{subfigure}
	\begin{subfigure}[]{.6\linewidth}
		\centering\scalebox{0.7}{\input{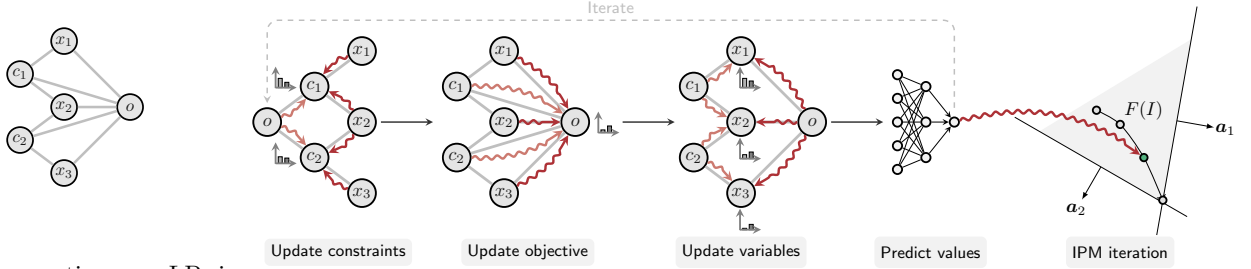}}
		\caption{IPM-MPNNs emulate interior-point methods.\\ \label{fig:overview}}
	\end{subfigure}
	\caption{Overview of our IPM-MPNN framework.}
	\vspace{-10pt}
\end{figure*}

\section{\uppercase{Background}}

In the following, we describe the necessary background.

\xhdr{Notation} Let $\Nb \coloneqq \{ 1,2,3, \dots \}$. For $n \geq 1$, let $[n] \coloneqq \{ 1, \dotsc, n \} \subset \Nb$. We use $\{\!\!\{ \dots\}\!\!\}$ to denote multisets, i.e., the generalization of sets allowing for multiple instances for each of its elements. A \new{graph} $G$ is a pair $(V(G),E(G))$ with \emph{finite} sets of
\new{vertices} or \new{nodes} $V(G)$ and \new{edges} $E(G) \subseteq \{ \{u,v\} \subseteq V(G) \mid u \neq v \}$. For ease of notation, we denote the edge $\{u,v\}$ in $E(G)$ by $(u,v)$ or $(v,u)$. Throughout the paper, we use standard notation, e.g., we denote the \new{neighborhood} of a node $v$ by $N(v)$, and so on; see~\cref{notation_app} for details. Moreover, let $\vc{x} \in \mathbb{R}^{1 \times d}$, then $\vc{D}(\vc{x})$ denotes the diagonal matrix with diagonal $\vc{x}$, $\vc{0}$ and $\vc{1}$ denote the vector of zero and ones, respectively, with an appropriate number of entries. By default, a vector $\vc{x} \in \mathbb{R}^{d}$ is a column vector.

\xhdr{Linear optimization problems} A \new{linear optimization problem} (LP) aims at optimizing a linear function over a feasible set described as the intersection of finitely many half-spaces, i.e., a polyhedron. We restrict our attention to feasible and bounded LPs. Formally, an instance $I$ of an LP is a tuple $(\vc{A},\vc{b},\vc{c})$, where $\vc{A}$ is a matrix in $\Qb^{m \times n}$, and $\vc{b}$ and $\vc{c}$ are vectors in $\Qb^{m}$ and $\Qb^{n}$, respectively.
We aim at finding a vector $\vc{x}^*$ in $\Qb^{n}$ that minimizes $\vc{c}\tran \vc{x}^*$ over the \emph{feasible set}
\begin{align}\label{eq:def-lp}
	\begin{split}
		F(I) = \{ \vc{x} \in \Qb^n \mid &\vc{A}_j \vc{x} \leq b_j \text{ for } j \in [m] \text{ and }  \\
		& x_i \geq 0 \text{ for } i \in [n]  \}.
	\end{split}
\end{align}
In practice, LPs are solved using the Simplex method or polynomial-time IPMs~\citep{nocedal2006numerical}.

We now detail the theoretical result that motivates our approach. We first summarize both interior point methods for linear optimization and MPNNs and then prove a theorem that relates the two.

\xhdr{Interior-point methods for linear optimization}
IPMs are algorithms for solving constrained optimization problems. They are particularly efficient for linear optimization, where they were first developed as a (polynomial-time) alternative to the Simplex methods. Variants of the algorithm differ in theoretical guarantees and empirical performance but revolve around the same core approach~\citep{Sha+2012}. In short, the LP to solve is replaced by a perturbed family of problems where a barrier penalty has replaced hard constraints with a parameter $\mu>0$. IPMs alternate between taking a Newton step to solve this perturbed problem and decreasing this parameter $\mu>0$, eventually converging to the optimal solution of the original problem.

For concreteness, we present two variants of the approach, an algorithm with theoretical guarantees and a practical algorithm that could be used in practice, following \citet[Chapter 14]{nocedal2006numerical} and \citet{gondzio2012}, respectively. In the next section, we will show that both algorithms can be connected to MPNNs.

The core idea of IPMs is as follows. First, we consider a perturbed version of the LP \eqref{eq:def-lp},
\begin{align}
	\underset{\vc{x} \in \Qb^n}{\min} & \vc{c}\tran \vc{x} - \mu[\vc{1}\tran \log(\vc{Ax}-\vc{b}) + \vc{1}\tran\log(\vc{x})]
	\label{eq:perturbed-lp}
\end{align}
for some $\mu>0$. By introducing the variables $s_i=\mu/x_i$, $r_j=\vc{A}_jx-b_j$ and $w_j=\mu/r_j$, the first-order optimality conditions for \eqref{eq:perturbed-lp} can be written as a system
\begin{align*}
	\vc{Ax}^* - \vc{r}^*            & = \vc{b}              \\
	\vc{A}\tran \vc{w}^* + \vc{s}^* & = \vc{c}              \\
	x^*_is^*_i                      & = \mu    & i \in [n], \\
	w^*_ir^*_i                      & = \mu    & j \in [m],
\end{align*}
with $\vc{x}^*, \vc{w}^*, \vc{s}^*, \vc{r}^* \geq \vc{0}$. Let $\sigma\in(0,1)$ be another hyperparameter. The two algorithms start from an initial positive point $(\vc{x}_0, \vc{w}_0, \vc{s}_0, \vc{r}_0)> \vc{0}$, and alternate between computing the Newton step for the perturbed problem \eqref{eq:perturbed-lp} at barrier parameter $\sigma\mu$
\begin{align*}
	\resizebox{1.05\columnwidth}{!}{$
			\begin{bmatrix}
				\vc{A}         & \vc{0}         & \vc{0}         & \vc{-I}        \\
				\vc{0}         & \vc{A}\tran    & \vc{I}         & \vc{0}         \\
				\vc{D}(\vc{s}) & \vc{0}         & \vc{D}(\vc{x}) & \vc{0}         \\
				\vc{0}         & \vc{D}(\vc{r}) & \vc{0}         & \vc{D}(\vc{w})
			\end{bmatrix}
			\begin{bmatrix}
				\Delta \vc{x} \\
				\Delta \vc{w} \\
				\Delta \vc{s} \\
				\Delta \vc{r}
			\end{bmatrix}
			=
			\begin{bmatrix}
				\vc{b} - \vc{Ax} + \vc{r}                              \\
				\vc{c} - \vc{A}\tran \vc{w} - \vc{s}                   \\
				\sigma\mu\vc{1}- \vc{D}(\vc{x}) \vc{D} (\vc{s}) \vc{1} \\
				\sigma\mu\vc{1}-\vc{D}(\vc{w})\vc{D}(\vc{r})\vc{1}
			\end{bmatrix},
		$}
\end{align*}
and taking a step in that direction with length $\alpha>0$, such that the resulting point $(\vc{x}', \vc{w}', \vc{s}', \vc{r}') = (\vc{x}, \vc{w}, \vc{s}, \vc{r}) + \alpha(\Delta \vc{x}, \Delta \vc{w}, \Delta \vc{s}, \Delta \vc{r})$ satisfies $(\vc{x}', \vc{w}', \vc{s}', \vc{r}') > \vc{0}$.

The above system can be simplified as follows. First, we can infer that
\begin{align}
	\Delta \vc{s} & = \sigma\mu\vc{D}(\vc{x})^{-1}\vc{1} - \vc{s} - \vc{D}(\vc{x})^{-1}\vc{D}(\vc{s})\Delta \vc{x} \label{eq:ipm-deltas}, \\
	\Delta \vc{r} & = \sigma\mu\vc{D}(\vc{w})^{-1}\vc{1} - \vc{r} - \vc{D}(\vc{w})^{-1}\vc{D}(\vc{r})\Delta \vc{w} \label{eq:ipm-deltar},
\end{align}
which implies that
\begin{align*}
	\vc{A}\Delta  \vc{x} +  \vc{D}(\vc{w})^{-1} \vc{D}(\vc{r})\Delta \vc{w}     & = \vc{b}-\vc{A}\vc{x}+\sigma\mu\vc{D}(\vc{w})^{-1}\vc{1},       \\
	\vc{A}\tran \Delta \vc{w} - \vc{D}(\vc{x})^{-1} \vc{D}(\vc{s})\Delta \vc{x} & = \vc{c}-\vc{A}\tran \vc{w}-\sigma\mu\vc{D}(\vc{x})^{-1}\vc{1}.
\end{align*}
Therefore, $ \Delta \vc{x} =$
\begin{align}
	                    & \hspace*{-27pt}\vc{D}(\vc{s})^{-1}\vc{D}(\vc{x})[\vc{A}\tran \Delta \vc{w} - \vc{c} + \vc{A}\tran \vc{w} + \sigma\mu\vc{D}(\vc{x})^{-1}\vc{1}] \label{eq:ipm-deltax} \\
	\vc{Q}\Delta \vc{w} & = \vc{b}-\vc{Ax}+\sigma\mu\vc{D}(\vc{w})^{-1}1 \label{eq:ipm-deltaw}                                                                                                 \\
	                    & + \vc{AD}(\vc{s})^{-1} \vc{D}(\vc{x})[\vc{c}-\vc{A}\tran \vc{w}-\sigma\mu\vc{D}(\vc{x})^{-1}\vc{1}] \notag
\end{align}
for $\vc{Q} = \vc{AD}(\vc{s})^{-1} \vc{D}(\vc{x})\vc{A}\tran + \vc{D}(\vc{w})^{-1} \vc{D}(\vc{r})$.  Thus, by solving the linear system in~\Cref{eq:ipm-deltaw}, we can find $\Delta \vc{w}$, then $\Delta \vc{x}$, $\Delta \vc{r}$, and $\Delta \vc{s}$ through Equations~\eqref{eq:ipm-deltas}-\eqref{eq:ipm-deltax}.

The two algorithms we consider only differ in how they compute $\mu$ and $\alpha$. The theoretical algorithm recomputes, at every iteration, $\mu=(\vc{x}\tran \vc{s}+\vc{w}\tran \vc{r})/(n+m)$, and chooses $\alpha$ to be the largest $\alpha<1$ such that $x'_is'_i\geq\gamma (\vc{x}'^{\mkern-1.5mu\mathsf{T}}  \vc{s}'^{\mkern-1.5mu} + \vc{w}'^{\mkern-1.5mu\mathsf{T}} \vc{r}')/(n+m)$ and $w'_jr'_j\geq\gamma(\vc{x}'^{\mkern-1.5mu\mathsf{T}} \vc{s}'+\vc{w}'^{\mkern-1.5mu\mathsf{T}} \vc{r}')/(n+m)$ for $i \in [n]$ and $j \in [m]$ for some hyperparameter $\gamma\in(0,1]$ \citep[Algorithm 14.2]{nocedal2006numerical}. The practical algorithm instead picks $\mu_0=(\vc{x}_0\tran \vc{s}_0+ \vc{w}_0\tran\vc{r}_0)/(n+m)$ initially, for $(\vc{x}_0, \vc{w}_0, \vc{s}_0, \vc{r}_0)$ the initial point, and thereafter decreases $\mu$ as $\mu'=\sigma\mu$ at every iteration, while choosing $\alpha$ to be $\alpha=0.99\alpha'$ for $\alpha$ the largest $\alpha>0$ such that $x'_is'_i>0$, $w'_jr'_j>0$. The two algorithms are summarized in Algorithms \ref{alg:ipm-theory} and \ref{alg:ipm-practice}.

\begin{algorithm}
	\caption{Theoretical IPM for LPs}
	\label{alg:ipm-theory}
	\begin{algorithmic}[1]
		\REQUIRE An LP instance $(\vc{A},\vc{b},\vc{c})$, a barrier reduction hyperparameter $\sigma\in(0,1)$, a neighborhood hyperparameter $\gamma\in(0,1]$ and initial values $(\vc{x}_0, \vc{w}_0, \vc{s}_0, \vc{r}_0)$ such that $\vc{Ax}_0-\vc{r}_0=\vc{b}$, $\vc{A}\tran \vc{w}_0+\vc{s}_0=\vc{c}$, $(\vc{x}_0, \vc{w}_0, \vc{s}_0, \vc{r}_0)>0$ and $\min_i x_{0i} w_{0i}\geq \gamma\mu_0$, $\min_i w_{0i}r_{0i} \geq \gamma\mu_0$ for $\mu_0=(\vc{x}_0\tran \vc{s}_0+(\vc{w}_0\tran)_0)/(n+m)$.

		\REPEAT
		\STATE $\mu\leftarrow (\vc{x}\trans \vc{s}+ \vc{w}\tran \vc{r})/(n+m)$ \label{eq:ipm-theory:mu}
		\STATE Compute $\Delta \vc{w}$ by solving the linear system \newline
		$\vc{Q}\Delta \vc{w} = \vc{b}-\vc{Ax}+\sigma\mu\vc{D}(\vc{w})^{-1}\vc{1}$ \newline
		\hspace*{2.5em} $+ \vc{A}\vc{D}(\vc{s})^{-1}\vc{D}(\vc{x})[\vc{c}-\vc{A}\tran \vc{w}-\sigma\mu\vc{D}(\vc{x})^{-1}\vc{1}]$ \newline
		for $\vc{Q}=\vc{A}\vc{D}(\vc{s})^{-1}\vc{D}(\vc{x})\vc{A}\tran+\vc{D}(\vc{w})^{-1}\vc{D}(\vc{r})$
		\label{eq:ipm-theory:deltaw}
		\STATE $\Delta \vc{x} \leftarrow \vc{D}(\vc{s})^{-1}\vc{D}(\vc{x})[\vc{A}\tran\Delta \vc{w} - \vc{c} + \vc{A}\tran\vc{w} + \sigma\mu\vc{D}(\vc{x})^{-1}\vc{1}]$ \label{eq:ipm-theory:deltax}
		\STATE $\Delta \vc{s} \leftarrow \sigma\mu\vc{D}(\vc{x})^{-1}\vc{1} - \vc{s} - \vc{D}(\vc{x})^{-1}\vc{D}(\vc{s})\Delta \vc{x}$ \label{eq:ipm-theory:deltas}
		\STATE $\Delta \vc{r} \leftarrow \sigma\mu\vc{D}(\vc{w})^{-1}\vc{1} - \vc{r} - \vc{D}(\vc{w})^{-1}\vc{D}(\vc{r})\Delta \vc{w}$ \label{eq:ipm-theory:deltar}
		\STATE Compute the largest $\alpha\in(0,1)$ such that
			{\footnotesize\begin{align*}
					 & \hspace{-10pt}\min_{i,j}\{(\vc{x}+\alpha\Delta \vc{x})_i(\vc{s}+\alpha\Delta \vc{s})_i,
					(\vc{w}+\alpha\Delta \vc{w})_j(\vc{r}+\alpha\Delta \vc{r})_j\}                                                                                                         \\
					 & \hspace{-10pt}\geq \gamma\frac{(\vc{x}+\alpha\Delta \vc{x})\tran (\vc{s}+\alpha\Delta \vc{s})+(\vc{w}+\alpha\Delta \vc{w})\tran (\vc{r}+\alpha\Delta \vc{r})}{n+m}.
				\end{align*}}
		\vspace{-5pt}
		\label{eq:ipm-theory:alpha}
		\STATE Update $(\vc{x},\vc{w},\vc{s},\vc{r}) \mathrel{+}= \alpha(\Delta \vc{x}, \Delta \vc{w}, \Delta \vc{s}, \Delta \vc{r})$
		\label{eq:ipm-theory:update}
		\UNTIL{convergence of $(\vc{x},\vc{w},\vc{s},\vc{r})$}
		\RETURN the point $\vc{x}$, which solves \ref{eq:def-lp}.
	\end{algorithmic}
\end{algorithm}

\begin{algorithm}
	\caption{Practical IPM for LPs}
	\label{alg:ipm-practice}
	\begin{algorithmic}[1]
		\REQUIRE An LP instance $(\vc{A},\vc{b},\vc{c})$, a barrier reduction hyperparameter $\sigma \in (0,1)$ and initial values $(\vc{x}_0, \vc{w}_0, \vc{s}_0, \vc{r}_0, \mu_0)$ such that $(\vc{x}_0, \vc{w}_0, \vc{s}_0, \vc{r}_0)>0$ and $\mu_0=(\vc{x}_0\tran \vc{s}_0+ \vc{w}_0\tran\vc{r}_0)/(n+m)$.

		\REPEAT
		\STATE Compute $\Delta \vc{w}$ by solving the linear system \newline
		$\vc{Q}\Delta \vc{w} = \vc{b}-\vc{Ax}+\sigma\mu\vc{D}(\vc{w})^{-1}\vc{1}$ \newline
		\hspace*{2.5em} $+ \vc{A}\vc{D}(\vc{s})^{-1}\vc{D}(\vc{x})[\vc{c}-\vc{A}\tran \vc{w}-\sigma\mu\vc{D}(\vc{x})^{-1}\vc{1}]$ \newline
		for $\vc{Q}=\vc{A}\vc{D}(\vc{s})^{-1}\vc{D}(\vc{x})\vc{A}\tran+\vc{D}(\vc{w})^{-1}\vc{D}(\vc{r})$.
		\label{eq:ipm-practice:deltaw}
		\STATE $\Delta \vc{x} \leftarrow \vc{D}(\vc{s})^{-1}\vc{D}(\vc{x})[\vc{A}\tran\Delta \vc{w} - \vc{c} + \vc{A}\tran\vc{w} + \sigma\mu\vc{D}(\vc{x})^{-1}\vc{1}]$
		\STATE $\Delta \vc{s} \leftarrow \sigma\mu\vc{D}(\vc{x})^{-1}\vc{1} - \vc{s} - \vc{D}(\vc{x})^{-1}\vc{D}(\vc{s})\Delta \vc{x}$
		\STATE $\Delta \vc{r} \leftarrow \sigma\mu\vc{D}(\vc{w})^{-1}\vc{1} - \vc{r} - \vc{D}(\vc{w})^{-1}\vc{D}(\vc{r})\Delta \vc{w}$
		\label{eq:ipm-practice:deltar}
		\STATE Find the largest $\alpha>0$ such that
		\vspace{-4pt}
		{\footnotesize\begin{align*}
				 & \min_{i,j}\{(\vc{x}+\alpha\Delta \vc{x})_i(\vc{s}+\alpha\Delta \vc{s})_i, \\
				 & \hspace{20pt}
				(\vc{w}+\alpha\Delta \vc{w})_j(\vc{r}+\alpha\Delta \vc{r})_j\} \geq 0
			\end{align*}}
		\label{eq:ipm-practice:alpha}
		\vspace{-18pt}
		\STATE Update $(\vc{x}, \vc{w}, \vc{s}, \vc{r}) \mathrel{+}= 0.99\alpha(\Delta \vc{x}, \Delta \vc{w}, \Delta \vc{s}, \Delta \vc{r})$
		\label{eq:ipm-practice:update}
		\vspace{-0pt}
		\STATE $\mu\leftarrow \sigma\mu$
		\label{eq:ipm-practice:mu}
		\UNTIL{convergence of $(\vc{x}, \vc{w}, \vc{s}, \vc{r})$}
		\RETURN the point $\vc{x}$, which solves \ref{eq:def-lp}.
	\end{algorithmic}
\end{algorithm}
Algorithm \ref{alg:ipm-theory} is guaranteed to converge to an $\epsilon$-accurate solution in $\cO((n+m)\log(1/\epsilon))$ iterations \cite[Theorem 14.3]{nocedal2006numerical}, that is, to a number of iterations proportional to the problem size. Algorithm \ref{alg:ipm-practice}, in contrast, does not come with any theoretical guarantees but is typical of practical IPM algorithms, which tend to converge in an almost constant number of iterations---usually within 30-40 iterations, irrespective of problem size \citep{gondzio2012, colombo2008further}.

\xhdr{Message-passing graph neural networks}
\label{sec:gnn}
Intuitively, MPNNs learn node features or attributes, i.e., a $d$-element real-valued vector, representing each node in a graph by aggregating information from neighboring nodes. Let $\mG=(G,\vc{L})$ be an attributed graph, following, \citet{Gil+2017} and \citet{Sca+2009}, in each layer, $t > 0$,  we update node attributes or features
\begin{equation*}\label{def:gnn}
	\vc{h}_{v}^\tup{t} \coloneqq
	\UPD^\tup{t}\Bigl(\vc{h}_{v}^\tup{t-1},\mathsf{MSG}^\tup{t} \bigl(\oms \vc{h}_{u}^\tup{t-1}
	\mid u\in N(v) \cms \bigr)\Bigr),
\end{equation*}
and $\vc{h}_{v}^\tup{0} \coloneqq \vc{L}_v$. Here,
the \new{message function} $\mathsf{MSG}^\tup{t}$ is a parameterized function, e.g., a neural network, mapping the multiset of neighboring node features to a single vectorial representation. We can easily adapt a message function to incorporate possible edge features or weights. Similarly, the \new{update function} $\mathsf{UPD}^\tup{t}$ is a parameterized function mapping the previous node features, and the output of $\mathsf{MSG}^\tup{t}$ to a single vectorial representation.

To adapt the parameters of the above functions, they are optimized end-to-end, usually through a variant of stochastic gradient descent, e.g.,~\citet{Kin+2015}, together with the parameters of a neural network used for classification or regression. In the following, we define a \new{message-passing step} as the application of a message and update function.

\section{\uppercase{Simulating IPM}s \uppercase{via MPNN}s}\label{sec:ipm-mpnn}
We now show that there exist MPNNs, with specific architecture and choices of parameters, such that~\cref{alg:ipm-theory,alg:ipm-practice} can be interpreted as inference over these MPNNs for a specific tripartite graph encoding the LP as input.

\xhdr{Representing LPs as graphs}
Let $I = (\vc{A},\vc{b},\vc{c})$ be an instance of LP. Similar to the setting of \citet{Ding2019}, we model the instances with an (undirected) \new{weighted tripartite graph} $G(I) \coloneqq (V(I), C(I), \{ o \},  E(I)_{\text{vc}}, E(I)_{\text{vo}}, E(I)_{\text{co}})$. Here, the node set $V(I) \coloneqq \{ v_i \mid i \in [n] \}$ represents the variables of $I$, the node set $C(I) \coloneqq \{ c_i \mid i \in [m] \}$ represents the constraints of $I$, and the node $o$ represents the objective. Further, the first edge set $E(I)_{\text{vc}}$ models the variable-constraint interaction, i.e., $E(I)_{\text{vc}} \coloneq \{ (v_i,c_j) \mid {A}_{ij} \neq 0 \}$, where each such edge $(v_i,c_j)$ is annotated with the weight ${A}_{ij}$. Further,  the objective node $o$ is connected to all other nodes in the graphs, i.e., $E(I)_{\text{vo}} \coloneqq \{ (o,v_i)  \mid v_i \in V(I) \}$ and $E(I)_{\text{co}} \coloneqq \{ (o,c_i)  \mid c_i \in C(I) \}$. Each edge $(o,v_i) \in E(I)_{\text{vo}}$ is annotated with the weight $c_i$. Similarly, each edge $(o,c_i) \in E(I)_{\text{co}}$ is annotated with the weight $b_i$. The resulting graph is illustrated in~\Cref{fig:lp}.

\xhdr{Theoretical results}
We now state the main results of this paper. To describe them, first notice that Algorithms \ref{alg:ipm-theory} and \ref{alg:ipm-practice} operate by taking an initial point $(\vc{x}_0, \vc{w}_0, \vc{s}_0, \vc{r}_0)$ and a duality measure $\mu_0$, and updating them after every iteration, yielding a sequence of points $(\vc{x}_t, \vc{w}_t, \vc{s}_t, \vc{r}_t)$ and duality measure $\mu_t$ for iterations $t>0$. The following result shows that \Cref{alg:ipm-theory} can be reproduced by a specific MPNN, in the sense that a fixed-depth MPNN can reproduce each of its iterations.
\begin{theorem}\label{thm:main}
	There exists an MPNN $f_{\mathsf{MPNN}, \mathsf{IPM1}}$ composed of $\cO(m)$ message-passing steps that reproduces an iteration of \Cref{alg:ipm-theory}, in the sense that for any LP instance $I=(\vc{A},\vc{b},\vc{c})$ and any iteration step $t\geq0$, $f_{\mathsf{MPNN}, \mathsf{IPM1}}$ maps the graph $G(I)$ carrying $[\vc{x}_t, \vc{s}_t]$ on the variable nodes and $[\vc{w}_t, \vc{r}_t]$ on the constraint nodes to the same graph $G(I)$ carrying  $[\vc{x}_{t+1}, \vc{s}_{t+1}]$ on the variable nodes and $[\vc{w}_{t+1}, \vc{r}_{t+1}]$ on the constraint nodes.
\end{theorem}
This implies, by composing several instances of $f_{\mathsf{MPNN}, \mathsf{IPM1}}$, that \Cref{alg:ipm-theory} can be simulated by an MPNN with a number of layers proportional to the number of iterations taken by the algorithm.
\begin{proposition}\label{thm:side}
	There exists an MPNN $f_{\mathsf{MPNN}, \mathsf{IPM2}}$ composed of $\cO(m)$ message-passing steps that reproduces each iteration of \Cref{alg:ipm-practice}, in the sense that for any LP instance $I=(\vc{A},\vc{b},\vc{c})$ and any iteration step $t\geq0$, $f_{\mathsf{MPNN}, \mathsf{IPM2}}$ maps the graph $G(I)$ carrying $[\vc{x}_t, \vc{s}_t]$ on the variable nodes, $[\vc{w}_t, \vc{r}_t]$ on the constraint nodes and $[\mu_t]$ on the objective node to the same graph $G(I)$ carrying  $[\vc{x}_{t+1}, \vc{s}_{t+1}]$ on the variable nodes, $[\vc{w}_{t+1}, \vc{r}_{t+1}]$ on the constraint nodes and $[\mu_{t+1}]$ on the objective node.
\end{proposition}
Similarly, this implies, by composing several instances of $f_{\mathsf{MPNN}, \mathsf{IPM2}}$, that \Cref{alg:ipm-practice} can be simulated by an MPNN with a number of layers proportional to the number of iterations taken by the algorithm.

\xhdr{Implication of the theoretical results} \cref{thm:main,thm:side} show that MPNNs are, in principle, capable of simulating modern IPMs. That is, they are capable of solving LPs to optimality. Hence, our findings shed light on the recent success of MPNN-based neural architectures by~\cite{Gas+2019} and similar approaches, which use MPNNs to mimic strong branching within the branch-and-bound framework for solving mixed-integer linear optimization problems. Moreover, in the following section, we derive MPNN architectures that act as lightweight proxies for solving LPs while being able to adapt to the given problem instance distributions. We verify their effectiveness empirically on real-world LP instances stemming from relaxing mixed-integer linear formulations of well-known combinatorial optimization problems.

\section{\uppercase{IPM-MPNN}s: \uppercase{MPNN}s \uppercase{for LP}s}\label{sec:ipm-mpnns}

Inspired by the theoretical alignment of IPMs and MPNNs derived above, we outline our IPM-MPNN framework, allowing for solving LP instances while adapting to a given problem instance distribution.

Given an LP instance $I$, we now outline an asynchronous MPNN operating on the tripartite graph $G(I)$; see~\cref{sec:ipm-mpnn}. Let $\vc{h}_c^{(t)} \in \Rb^d$, $d > 0$, be the node features of a constraint node $c \in C(I)$ at iteration $t > 0$, and let  $\vc{h}_v^{(t)} \in \Rb^d$ and $\vc{h}_o^{(t)} \in \Rb^d$ be the node features of a variable node $v \in V(I)$ and the objective node $o$ at iteration $t$, respectively. Moreover, let $\vc{e}_{co},\vc{e}_{vc},\vc{e}_{vo}$ denote the edge weights.
Initially, at $t=0$, we set the node features by applying a linear mapping to the raw node features $\vc{x}_v$, $\vc{x}_c$ or $\vc{x}_o$, extracted from the instance $I$; see~\cref{sec:experiments} for details.

All three node types are updated in three separate update passes. In the first pass, we update the embeddings of constraint nodes from the embeddings of the variable nodes and of the objective node. That is, let $c \in C(I)$ be a constraint node and let $t > 0$, then
\begin{equation*}
	\begin{aligned}
		\vc{h}_c^{(t)} \coloneqq\, & \textsf{UPD}^{(t)}_{\text{c}}\Bigl[ \vc{h}_c^{(t-1)}, \textsf{MSG}^{(t)}_{\text{o} \rightarrow \text{c}}\left(\vc{h}_{o}^{(t-1)}, \vc{e}_{oc} \right),             \\
		                           & \textsf{MSG}^{(t)}_{\text{v} \rightarrow \text{c}}\left(\{\!\!\{ (\vc{h}_v^{(t-1)}, \vc{e}_{vc}) \mid v \in {N}\left(c \right) \cap V(I) \}\!\!\}  \right) \Bigr].
	\end{aligned}
\end{equation*}
Here, the parameterized message function $\textsf{MSG}^{(t)}_{\text{v} \rightarrow \text{c}}$ maps a multiset of vectors, i.e., variable node features and corresponding edge features $\vc{e}_{{vc}}$, to a vector in $\Rb^d$. Similarly, the parameterized function $\textsf{MSG}_{\text{o} \rightarrow \text{c}}$ maps the current node features of the objective node and edge features $\vc{e}_{{oc}}$ to a vector in $\Rb^d$. Finally, the parameterized function $\textsf{UPD}^{(t)}_{\text{c}}$ maps the constraint node's previous features, the outputs of $\textsf{MSG}^{(t)}_{\text{o} \rightarrow \text{c}}$ and $\textsf{MSG}^{(t)}_{\text{v} \rightarrow \text{c}}$ to vector in $\Rb^d$.

Next, similarly to the above, we update the objective node's features depending on variable and constraint node features,
\begin{equation*}
	\begin{aligned}
		\vc{h}_{\text{o}}^{(t)} \coloneqq\, & \textsf{UPD}^{(t)}_{\text{o}}\Bigl[ \vc{h}_o^{(t-1)}, \textsf{MSG}^{(t)}_{\text{c} \rightarrow \text{o}}\left(\{\!\!\{ \vc{h}_c^{(t)}, \vc{e}_{co} \mid c \in C(I) \}\!\!\}  \right), \\
		                                    & \textsf{MSG}^{(t)}_{\text{v} \rightarrow \text{o}}\left( \{\!\!\{ \vc{h}_v^{(t-1)}, \vc{e}_{vo} \mid v \in V(I) \}\!\!\}\right)\Bigr].
	\end{aligned}
\end{equation*}

Finally, analogously to the update of the constraint nodes' features, we update the representation of a variable  node $v \in V(I)$ from the constraints nodes and objective node,
\begin{equation*}
	\begin{aligned}
		\vc{h}_v^{(t)} \coloneqq\, & \textsf{UPD}^{(t)}_{\text{v}}\Bigl[ \vc{h}_v^{(t-1)}, \textsf{MSG}^{(t)}_{\text{o} \rightarrow \text{v}}\left(\vc{h}_o^{(t)}, \vc{e}_{ov} \right),          \\
		                           & \textsf{MSG}^{(t)}_{\text{c} \rightarrow \text{v}}\left( \{\!\!\{ \vc{h}_c^{(t)}, \vc{e}_{cv} \mid c \in N \left(v \right) \cap V(C) \}\!\!\}\right)\Bigr].
	\end{aligned}
\end{equation*}
The entire process is executed asynchronously, in that the nodes updated later incorporate the most recent features updates from preceding nodes. We map each variable node feature $\vc{h}_v^{(t)}$ to $\mathsf{MLP}(\vc{h}_v^{(t)}) \in \Rb,$ where $\mathsf{MLP}$ is a multi-layer perceptron, and concatenate the resulting real numbers over all variable nodes column-wise, resulting in the final prediction $\vc{z}^{(t)} \in \Rb^n$.

In the experiments in \Cref{sec:experiments}, we probe various message-passing layers to express the various message and update functions. Concretely, we leverage the GCN \citep{Kip+2017}, GIN \citep{Xu+2018b}, and GEN \citep{li2020deepergcn} MPNN layers, respectively; see \cref{sec:conv_form} for details.

We train the above IPM-MPNN architecture, i.e., adapt its parameters, in a supervised fashion. Below, we outline the three components constituting our training loss function. To that, let $(\vc{A}, \vc{b}, \vc{c})$ be an LP instance.

\xhdr{Variable supervision} As discussed above, our IPM-MPNN aims to simulate the solving steps provided by a standard IPM. Thus, aligned with~\cref{thm:main,thm:side}, a perfectly parameterized MPNN is expected to follow the steps without deviation for all $T$ iterations. We maintain the intermediate outputs $\vc{z}^{(t)} \in \Rb^n$ of each MPNN layer and calculate the mean squared error (MSE) loss between every pair of the expert solution $\vc{y}^{(t)} \in \Rb^n$ and MPNN prediction $\vc{z}^{(t)}$. Moreover, we introduce a step decay factor $\alpha \in [0, 1]$ so that early steps play a less important role, resulting in the following loss function,
\begin{equation}
	\label{eq:main_loss}
	\mathcal{L}_{\text{var}} \coloneqq \frac{1}{N} \sum_{i=1}^N \sum_{t=1}^T \alpha^{T-t} {\lVert \vc{y}_{i}^{(t)} - \vc{z}_{i}^{(t)} \rVert}_2^2,
\end{equation}
where $N$ denotes the number of training samples.

\xhdr{Objective supervision} We use regularization on the prediction regarding the ground-truth objective values at every step. Empirically, this regularization term helps with convergence and helps finding more feasible solutions that minimize the objective in case the LP instance has multiple solutions. Note that we do not predict the objective directly with our MPNNs but calculate it via $\vc{c}\tran \vc{z}^{(t)}$ instead. Suppose the ground-truth values are given by $\vc{c}\tran \vc{y}^{(t)}$, we have
\begin{equation}
	\mathcal{L}_{\text{obj}} \coloneqq \frac{1}{N} \sum_{i=1}^N \sum_{t=1}^T \alpha^{T-t} \left[ \vc{c}\tran \big(\vc{y}_{i}^{(t)} - \vc{z}_{i}^{(t)} \big) \right]^2.
\end{equation}

\xhdr{Constraint supervision} Finally, we aim for the IPM-MPNN to predict an optimal solution regarding the objective value while satisfying all the constraints. To that, we introduce a regularization penalizing constraint violations, i.e.,
\begin{equation}
	\mathcal{L}_{\text{cons}} \coloneqq \frac{1}{N} \sum_{i=1}^N \sum_{t=1}^T \alpha^{T-t} {\lVert \mathsf{ReLU}(\vc{A}_{i} \vc{z}_i^{(t)} - \vc{b}_{i}) \rVert}_2^2.
\end{equation}

Finally, we combine the above three loss terms into the loss function
\begin{equation}
	\label{eq:full_loss}
	\mathcal{L} \coloneqq w_{\text{var}} \mathcal{L}_{\text{var}} + w_{\text{obj}} \mathcal{L}_{\text{obj}} + w_{\text{cons}}\mathcal{L}_{\text{cons}},
\end{equation}
where we treat $w_{\text{var}}, w_{\text{obj}},$ and $w_{\text{cons}} > 0$ as hyperparamters.

At test time, given an LP instance $I$, we construct the tripartite graph $G(I)$ as outlined above and use the trained  MPNN to predict the variables' values.

\begin{table*}[t]
	\caption{Results of our proposed IPM-MPNNs (\cmark) versus bipartite representation ablations (\xmark). We report the relative objective gap and the constraint violation, averaged over all three runs. We print the best results per target in bold.}
	\label{tab: main_table}
	\begin{center}
		\resizebox{0.9\textwidth}{!}{
			\begin{tabular}{ccccccc|cccc}
				\toprule
				                                                                              & {\multirow{2}{*}{\scriptsize Tri.}}   & {\multirow{2}{*}{\scriptsize MPNN}}  & \multicolumn{4}{c}{\textbf{Small instances}} & \multicolumn{4}{c}{\textbf{Large instances}}                                                                                                                                                                                                                                           \\

				                                                                              &                                       &                                      & Setcover                                     & Indset                                       & Cauc                                 & Fac                                  & Setcover                            & Indset                                 & Cauc                                & Fac                                  \\

				\midrule
				{\multirow{6}{*}{\rotatebox[origin=c]{90}{\scriptsize Objective gap [\%]}}}   & {\multirow{3}{*}{\scriptsize \cmark}}

				                                                                              & GEN                                   & \textbf{0.319\scriptsize$\pm$0.020}  & 0.119\scriptsize$\pm$0.003                   & \textbf{0.612\scriptsize$\pm$0.049}          & \textbf{0.549\scriptsize$\pm$0.112}  & 0.629\scriptsize$\pm$0.086           & 0.158\scriptsize$\pm$0.035          & \textbf{0.306\scriptsize$\pm$0.047}    & \textbf{0.747\scriptsize$\pm$0.083}                                        \\
				                                                                              &                                       & GCN                                  & 0.418\scriptsize$\pm$0.008                   & \textbf{0.103\scriptsize$\pm$0.006}          & 0.682\scriptsize$\pm$0.029           & 0.578\scriptsize$\pm$0.015           & \textbf{0.420\scriptsize$\pm$0.047} & \textbf{0.094\scriptsize$\pm$0.005}    & 0.407\scriptsize$\pm$0.038          & 0.914\scriptsize$\pm$0.141           \\
				                                                                              &                                       & GIN                                  & 0.478\scriptsize$\pm$0.038                   & 0.146\scriptsize$\pm$0.011                   & 0.632\scriptsize$\pm$0.036           & 0.810\scriptsize$\pm$0.221           & 0.711\scriptsize$\pm$0.115          & 0.126\scriptsize$\pm$0.021             & 0.378\scriptsize$\pm$0.052          & 0.911\scriptsize$\pm$0.132           \\
				\cmidrule{2-11}

				                                                                              & {\multirow{3}{*}{\scriptsize \xmark}}
				                                                                              & GEN                                   & 8.310\scriptsize$\pm$1.269           & 0.735\scriptsize$\pm$0.032                   & 1.417\scriptsize$\pm$0.009                   & 2.976\scriptsize$\pm$0.013           & 15.170\scriptsize$\pm$6.844          & 0.320\scriptsize$\pm$0.008          & 0.851\scriptsize$\pm$0.122             & 2.531\scriptsize$\pm$0.025                                                 \\
				                                                                              &                                       & GCN                                  & 5.523\scriptsize$\pm$0.133                   & 0.639\scriptsize$\pm$0.009                   & 1.394\scriptsize$\pm$0.081           & 3.031\scriptsize$\pm$0.059           & 6.092\scriptsize$\pm$0.456          & 0.298\scriptsize$\pm$0.009             & 0.766\scriptsize$\pm$0.093          & 2.535\scriptsize$\pm$0.034           \\
				                                                                              &                                       & GIN                                  & 5.592\scriptsize$\pm$0.179                   & 0.634\scriptsize$\pm$0.021                   & 1.202\scriptsize$\pm$0.016           & 2.996\scriptsize$\pm$0.031           & 5.835\scriptsize$\pm$1.917          & 0.290\scriptsize$\pm$0.005             & 0.810\scriptsize$\pm$0.140          & 2.660\scriptsize$\pm$0.062           \\

				\midrule

				{\multirow{6}{*}{\rotatebox[origin=c]{90}{\scriptsize Constraint violation}}} & {\multirow{3}{*}{\scriptsize \cmark}}
				                                                                              & GEN                                   & \textbf{0.002\scriptsize$\pm$0.0002} & 0.0006\scriptsize$\pm$0.00003                & 0.003\scriptsize$\pm$0.0007                  & 0.002\scriptsize$\pm$0.001           & 0.009\scriptsize$\pm$0.001           & 0.0015\scriptsize$\pm$0.0003        & \textbf{0.0004\scriptsize $\pm$0.0002} & 0.002\scriptsize$\pm$0.001                                                 \\
				                                                                              &                                       & GCN                                  & 0.002\scriptsize$\pm$0.001                   & \textbf{0.0003\scriptsize$\pm$0.0001}        & 0.002\scriptsize$\pm$0.00007         & \textbf{0.002\scriptsize$\pm$0.0002} & 0.009\scriptsize$\pm$0.001          & \textbf{0.0005\scriptsize$\pm$0.00004} & 0.001\scriptsize$\pm$0.0005         & \textbf{0.001\scriptsize$\pm$0.0004} \\
				                                                                              &                                       & GIN                                  & 0.004\scriptsize$\pm$0.001                   & 0.0006\scriptsize$\pm$0.00008                & \textbf{0.001\scriptsize$\pm$0.0001} & 0.002\scriptsize$\pm$0.0005          & \textbf{0.008\scriptsize$\pm$0.002} & 0.0006\scriptsize$\pm$0.0001           & 0.002\scriptsize$\pm$0.0008         & 0.002\scriptsize$\pm$0.0007          \\
				\cmidrule{2-11}

				                                                                              & {\multirow{3}{*}{\scriptsize \xmark}}
				                                                                              & GEN                                   & 0.181\scriptsize$\pm$0.023           & 0.006\scriptsize$\pm$0.0003                  & 0.006\scriptsize$\pm$0.001                   & 0.011\scriptsize$\pm$0.004           & 0.309\scriptsize$\pm$0.025           & 0.004\scriptsize$\pm$0.0002         & 0.006\scriptsize$\pm$0.001             & 0.003\scriptsize$\pm$0.001                                                 \\
				                                                                              &                                       & GCN                                  & 0.207\scriptsize$\pm$0.006                   & 0.004\scriptsize$\pm$0.001                   & 0.002\scriptsize$\pm$0.001           & 0.006\scriptsize$\pm$0.0003          & 0.267\scriptsize$\pm$0.049          & 0.003\scriptsize$\pm$0.0004            & 0.004\scriptsize$\pm$0.001          & 0.002\scriptsize$\pm$0.0003          \\
				                                                                              &                                       & GIN                                  & 0.211\scriptsize$\pm$0.007                   & 0.003\scriptsize$\pm$0.0002                  & 0.003\scriptsize$\pm$0.001           & 0.008\scriptsize$\pm$0.002           & 0.236\scriptsize$\pm$0.014          & 0.003\scriptsize$\pm$0.0004            & 0.004\scriptsize$\pm$0.002          & 0.003\scriptsize$\pm$0.0002          \\

				\bottomrule
			\end{tabular}
		}
	\end{center}
\end{table*}

\section{\uppercase{Experimental Study}}\label{sec:experiments}

Here, we empirically evaluate the ability of IPM-MPNNs to predict the optimal solutions of LPs. In particular, we aim to answer the following questions.

\xhdr{Q1} Can MPNNs properly imitate the performance of IPM solvers in practice?\\
\xhdr{Q2} How is MPNNs' performance compared with competing neural-network-based solutions?\\
\xhdr{Q3} What advantage does our MPNN architecture hold compared with traditional IPM solvers?\\
\xhdr{Q4} Do MPNNs generalize to instances larger than seen during training?

Our experimental results are reproducible with the code available at \url{https://github.com/chendiqian/IPM_MPNN}.

\xhdr{Datasets} We obtain LP instances from mixed-integer optimization instances by dropping the integrality constraints over variables. Following~\cite{Gas+2019}, we use four classes of problems, namely \new{set covering}, \new{maximal independent set}, \new{combinatorial auction}, and \new{capacitated facility location}. For each problem type, we generate small- and large-size instances. We describe the datasets and our parameters for dataset generation in \cref{app:dataset}. From each LP instance $I = (\vc{A}, \vc{b}, \vc{c})$, we generate a tripartite graph $G(I)$, following~\cref{sec:ipm-mpnn}, and construct initial node and edge features as follows. We indicate the $i$th row of the constraint matrix $\vc{A}$ as $\vc{A}_i$ and the $j$th column as $\vc{A}_{\cdot, j}$. For a given variable node $v_j \in V(I)$, its initial node features are set to the mean and standard deviation of the column vector $\vc{A}_{\cdot,j}$, resulting in two features. Analogously, for a constraint node $c_i \in C(I)$, we derive initial features from the statistical properties of the row vector $\vc{A}_i$. For the objective node $o$, the features are characterized in a corresponding manner using the vector $\vc{c}$. In the supervised learning regime, it is mandatory to have ground truth labels to guide the model predictions. For our IPM-MPNNs, the outputs of each layer $t$ are the prediction of the variables' value $\vc{z}^{(t)}$. Consequently, we utilize the intermediate variable values $\vc{y}^{(t)}$, as provided by the solver, to serve as our ground truth. However, to prevent the MPNN from becoming excessively deep, we sample the ground truth steps, i.e., we adopt an equidistant sampling strategy for the solver's steps, ensuring that the number of sampled steps aligns with the depth of the corresponding MPNN. We split the graph datasets into train, validation, and test splits with ratios $0.8, 0.1$, and $0.1$. We conducted the experiments by evaluating each dataset multiple times, performing three independent runs with distinct random seeds. The reported results represent the average numbers over the test set across these runs and the corresponding standard deviations. We executed all experiments on a single NVIDIA A100 80GB GPU card—see~\cref {sec:hyperparam} for training-related hyperparameters.

\xhdr{IPM-MPNNs' ability to solve LPs (Q1)} We collect and organize the results of our MPNN method in \cref{tab: main_table}. We report the numbers on the four types of relaxed MILP instances, each with small and large sizes. As explained in~\Cref{sec:ipm-mpnns}, we leverage three types of MPNN layers, GCN \citep{Kip+2017}, GIN \citep{Xu+2018b}, and GEN \citep{li2020deepergcn}, as the backbone of our MPNN architectures. We split the table into two main parts, namely the mean absolute relative objective gap
\begin{equation*}
	\frac{1}{N} \sum_{i=1}^N \left| \frac{\vc{c}\tran \big(\vc{y}_i^{(T)} - \vc{z}_i^{(T)} \big)}{\vc{c}\tran \vc{y}_i^{(T)}} \right| \times 100
\end{equation*}
of the last step $T$ over the test set, and the mean absolute constraint violation of the last step
\begin{equation*}
	\frac{1}{N} \sum_{i=1}^N \frac{1}{m_i} {\lVert \mathsf{reLU} (\vc{A}_i \vc{z}_i^{(T)} - \vc{b}_i) \rVert}_1,
\end{equation*}
where the normalization term $m_i$ is the number of constraints of the $i$th instance. As seen, our IPM-MPNN architectures consistently align with the IPM solver at the last converged step, with marginal constraint violation. Moreover, the relative objective gaps of our method are all under $1\%$. The GCN-based IPM-MPNNs perform best on large maximal independent set relaxation instances at $0.094\pm0.005 \%$. Overall, GEN and GCN perform better than GIN layers among the three graph convolutions. The absolute constraint violations are mostly at the $1 \times 10^{-3}$ level, with the best ($0.0003\pm0.0001$) achieved by GCN on small maximal independent set relaxation instances.

\xhdr{Baselines (Q2)} To answer question \textbf{Q2}, we compare IPM-MPNNs to two baselines. First, we compare our IPM-MPNNs to \citet{Che+2023}, where the authors proposed encoding LP instances as bipartite graphs. In \cref{tab: main_table}, the rows with a cross mark (\xmark) on the left are the bipartite baselines. As seen in this table, our IPM-MPNNs outperform the bipartite architectures in all instances with all types of MPNNs. Most relative objective gaps lie above $1\%$, except on the maximal independent set and large combinatorial auction relaxation instances. The largest gap is observed on the small set covering relaxation instances by GEN, with the baseline reporting as much as $26.1\times$ higher constraint violation number. Hence, our results indicate that IPM-MPNNs's tripartite representation is crucial. We also compare our IPM-MPNNs to a neural-ODE-based approach \citep{Wu+2023}. Since their approach is quite expensive during training by means of both time and GPU memory, we generate $\num{1000}$ mini-sized instances. Due to the architecture-agnostic property of their approach, we embed our MPNNs in their training pipeline. When presenting the runtime and GPU usage, we use the same batch size as the baseline method on our MPNN approach for a fair comparison, even though our method can scale to a much larger batch size in practice. We report results in \cref{tab:ode_baseline}. Taking the GEN layer as an example, our method shows consistently better results than the neural-ODE baseline method by reaching at most $14.2 \times$ lower relative objective gap, $8.6 \times$ faster training, and $300.8 \times $ less memory on the mini-sized capacitated facility location relaxation instances.

\begin{table}[t!]
	\caption{Comparing  between IPM-MPNNs  and the \citet{Wu+2023} method on $\num{1\,000}$ mini-sized instances. We report the average relative objective gap, constraint violation, training time over three runs, and maximal GPU memory allocated. We print the best results per target in bold.}
	\label{tab:ode_baseline}
	\begin{center}
		\resizebox{0.9\columnwidth}{!}{
			\begin{tabular}{ccccccc}
				\toprule

				                                                             & \textbf{Method}         & MPNN                                & Setcover                            & Indset                               & Cauc                                & Fac                                  \\

				\midrule
				{\multirow{6}{*}{\rotatebox[origin=c]{90}{Obj. gap [\%]}}}   & {\multirow{3}{*}{ODE}}

				                                                             & GEN                     & 14.915\scriptsize$\pm$0.425         & 6.225\scriptsize$\pm$0.097          & 13.845\scriptsize$\pm$0.554          & 20.560\scriptsize$\pm$0.059                                                \\
				                                                             &                         & GCN                                 & 14.545\scriptsize$\pm$0.055         & 6.148\scriptsize$\pm$0.071           & 12.945\scriptsize$\pm$0.385         & 20.690\scriptsize$\pm$0.037          \\
				                                                             &                         & GIN                                 & 15.050\scriptsize$\pm$0.228         & 6.474\scriptsize$\pm$0.114           & 13.470\scriptsize$\pm$1.145         & 21.010\scriptsize$\pm$0.529          \\
				\cmidrule{2-7}

				                                                             & {\multirow{3}{*}{Ours}}
				                                                             & GEN                     & 2.555\scriptsize$\pm$0.122          & 1.580\scriptsize$\pm$0.095          & \textbf{2.733\scriptsize$\pm$0.074}  & 1.449\scriptsize$\pm$0.255                                                 \\
				                                                             &                         & GCN                                 & \textbf{2.375\scriptsize$\pm$0.062} & 1.447\scriptsize$\pm$0.152           & 2.769\scriptsize$\pm$0.091          & 1.478\scriptsize$\pm$0.154           \\
				                                                             &                         & GIN                                 & 2.740\scriptsize$\pm$0.3184         & \textbf{1.404\scriptsize$\pm$0.153}  & 2.847\scriptsize$\pm$0.091          & \textbf{1.328\scriptsize$\pm$0.201}  \\

				\midrule

				{\multirow{6}{*}{\rotatebox[origin=c]{90}{Constraint vio.}}} & {\multirow{3}{*}{ODE}}

				                                                             & GEN                     & 0.072\scriptsize$\pm$0.006          & 0.046\scriptsize$\pm$0.002          & 0.025\scriptsize$\pm$0.008           & 0.020\scriptsize$\pm$0.001                                                 \\
				                                                             &                         & GCN                                 & 0.049\scriptsize$\pm$0.012          & 0.048\scriptsize$\pm$0.008           & 0.025\scriptsize$\pm$0.0002         & 0.020\scriptsize$\pm$0.0005          \\
				                                                             &                         & GIN                                 & 0.064\scriptsize$\pm$0.005          & 0.043\scriptsize$\pm$0.008           & 0.024\scriptsize$\pm$0.005          & 0.014\scriptsize$\pm$0.004           \\
				\cmidrule{2-7}
				                                                             & {\multirow{3}{*}{Ours}}
				                                                             & GEN                     & \textbf{0.023\scriptsize$\pm$0.002} & 0.005\scriptsize$\pm$0.0001         & 0.015\scriptsize$\pm$0.003           & 0.013\scriptsize$\pm$0.003                                                 \\
				                                                             &                         & GCN                                 & 0.030\scriptsize$\pm$0.003          & 0.005\scriptsize$\pm$0.0006          & 0.017\scriptsize$\pm$0.002          & \textbf{0.005\scriptsize$\pm$0.0006} \\
				                                                             &                         & GIN                                 & 0.023\scriptsize$\pm$0.005          & \textbf{0.005\scriptsize$\pm$0.0003} & \textbf{0.014\scriptsize$\pm$0.001} & 0.006\scriptsize$\pm$0.0006          \\

				\midrule

				{\multirow{6}{*}{\rotatebox[origin=c]{90}{Time [s]}}}        & {\multirow{3}{*}{ODE}}

				                                                             & GEN                     & 47.829                              & 51.283                              & 63.068                               & 96.298                                                                     \\
				                                                             &                         & GCN                                 & 57.196                              & 80.133                               & 79.606                              & 34.297                               \\
				                                                             &                         & GIN                                 & 55.918                              & 64.628                               & 39.904                              & 62.448                               \\
				\cmidrule{2-7}
				                                                             & {\multirow{3}{*}{Ours}}
				                                                             & GEN                     & 10.177                              & 9.617                               & 9.946                                & 11.124                                                                     \\
				                                                             &                         & GCN                                 & 18.964                              & 8.688                                & \textbf{7.368}                      & \textbf{8.834}                       \\
				                                                             &                         & GIN                                 & \textbf{6.042}                      & \textbf{8.096}                       & 8.881                               & 10.771                               \\

				\midrule

				{\multirow{6}{*}{\rotatebox[origin=c]{90}{Memory (GB)}}}     & {\multirow{3}{*}{ODE}}

				                                                             & GEN                     & 16.455                              & 25.931                              & 23.354                               & 44.520                                                                     \\
				                                                             &                         & GCN                                 & 16.489                              & 34.003                               & 23.805                              & 10.640                               \\
				                                                             &                         & GIN                                 & 18.238                              & 30.101                               & 13.482                              & 24.713                               \\
				\cmidrule{2-7}
				                                                             & {\multirow{3}{*}{Ours}}
				                                                             & GEN                     & \textbf{0.091}                      & 0.088                               & 0.101                                & 0.148                                                                      \\
				                                                             &                         & GCN                                 & 0.201                               & 0.134                                & \textbf{0.069}                      & \textbf{0.142}                       \\
				                                                             &                         & GIN                                 & 0.094                               & \textbf{0.073}                       & 0.148                               & 0.187                                \\

				\bottomrule
			\end{tabular}
		}
	\end{center}
\end{table}

\xhdr{Inference time profiling (Q3)} We also compare IPM-MPNN's performance to exact IPM solvers. Thereto, we compare the time required to solve an instance between traditional solvers, namely SciPy's IPM solver and a Python-based custom-build one, and our IPM-MPNN. We run the solvers and our MPNNs on the test set of each dataset and report the mean and standard deviation in seconds. According to \cref{tab:solver_time}, our MPNN clearly outperforms the SciPy IPM solver on all large instances. Further, IPM-MPNNs beat our Python-based IPM solver described in~\cref{alg:ipm-practice} on all instances. It is worth noting that both IPM solvers exhibit sensitivity to problem sizes. For example, the SciPy solver sees a performance degradation of approximately $65.0 \times$ when transitioning from small to large set covering problems. In contrast, our MPNN method demonstrates a more consistent behavior across varying problem sizes, showing only $1.2\times$ slowdown for the analogous instances. Hence, we can positively answer question \textbf{Q3}.

\begin{table}[t!]
	\caption{Comparing IPM-MPNNs' inference time to SciPy's IPM implementation and our Python-based IPM solver. We report mean and standard deviation in seconds over three runs. We print the best results per target in bold.}
	\label{tab:solver_time}
	\begin{center}
		\resizebox{0.9\columnwidth}{!}{
			\begin{tabular}{cccccc}
				\toprule
				\textbf{Instances} & SciPy Solver                        & Our Solver                 & GEN                                 & GCN                        & GIN                                 \\
				\midrule
				Small setcover     & \textbf{0.006\scriptsize$\pm$0.004} & 0.071\scriptsize$\pm$0.015 & 0.033\scriptsize$\pm$0.001          & 0.029\scriptsize$\pm$0.001 & 0.017\scriptsize$\pm$0.001          \\
				Large setcover     & 0.390\scriptsize$\pm$0.098          & 3.696\scriptsize$\pm$2.141
				                   & 0.033\scriptsize$\pm$0.001          & 0.030\scriptsize$\pm$0.001 & \textbf{0.021\scriptsize$\pm$0.001}                                                                    \\
				\midrule
				Small indset       & \textbf{0.008\scriptsize$\pm$0.067} & 0.089\scriptsize$\pm$0.024 & 0.033\scriptsize$\pm$0.001          & 0.031\scriptsize$\pm$0.002 & 0.021\scriptsize$\pm$0.001          \\
				Large indset       & 0.226\scriptsize$\pm$0.087          & 1.053\scriptsize$\pm$0.281 & 0.033\scriptsize$\pm$0.002          & 0.030\scriptsize$\pm$0.001 & \textbf{0.021\scriptsize$\pm$0.001} \\
				\midrule
				Small cauc         & \textbf{0.012\scriptsize$\pm$0.005} & 0.151\scriptsize$\pm$0.035 & 0.033\scriptsize$\pm$0.001          & 0.028\scriptsize$\pm$0.001 & 0.021\scriptsize$\pm$0.001          \\
				Large cauc         & 0.282\scriptsize$\pm$0.065          & 3.148\scriptsize$\pm$0.880 & 0.033\scriptsize$\pm$0.001          & 0.029\scriptsize$\pm$0.001 & \textbf{0.021\scriptsize$\pm$0.001} \\
				\midrule
				Small fac          & \textbf{0.017\scriptsize$\pm$0.011} & 2.025\scriptsize$\pm$1.854 & 0.029\scriptsize$\pm$0.001          & 0.029\scriptsize$\pm$0.001 & 0.022\scriptsize$\pm$0.001          \\
				Large fac          & 0.732\scriptsize$\pm$0.324          & 6.229\scriptsize$\pm$2.672 & 0.030\scriptsize$\pm$0.001          & 0.031\scriptsize$\pm$0.001 & \textbf{0.022\scriptsize$\pm$0.001} \\
				\bottomrule
			\end{tabular}
		}
	\end{center}
\end{table}

\xhdr{Size generalization (Q4)} We investigate the possibility of generalizing our pre-trained MPNNs to larger instances than encountered during training. To that, we generate new sets of novel instances, each with the same number of instances as the test set. In \cref{tab:size_gen} in the appendix, we list our training instance sizes and a collection of test instance sizes. Taking GEN on set covering relaxation problems as an example, as the inference size grows, the relative objective gap increases a bit, while the constraint violations are overall worse than on the training set. Notably, for the training size [600,800], which is at least $1.37 \times$ larger than the training instances, the objective gap is merely $0.28 \%$ worse than [500,700] size. For question \textbf{Q4}, we can conclude that our pre-trained MPNNs can generalize to unseen, larger instances to some extent.

\section{\uppercase{Conclusion}}
In summary, our study establishes a strong connection between  (MPNNs and IPMs for LPs. We have shown that MPNNs can effectively simulate IPM iterations, revealing their potential in emulating strong branching within the branch-and-bound framework, as demonstrated by~\cite{Gas+2019}. In addition, building on this connection, we proposed IPM-MPNNs for learning to solve large-scale LP instances approximately, surpassing neural baselines and exact IPM solvers in terms of solving time across various problem domains. Looking forward,  promising avenues for further research involve expanding our theoretical framework to encompass a broader range of convex optimization problems.

\subsubsection*{Acknowledgments}
CQ and CM are partially funded by a DFG Emmy Noether grant (468502433) and RWTH Junior Principal Investigator Fellowship under Germany's Excellence Strategy.

\bibliography{references}

\appendix
\onecolumn

\section{Extended notation}\label{notation_app}
A \new{graph} $G$ is a pair $(V(G),E(G))$ with \emph{finite} sets of
\new{vertices} or \new{nodes} $V(G)$ and \new{edges} $E(G) \subseteq \{ \{u,v\} \subseteq V(G) \mid u \neq v \}$. An \new{attributed graph} $G$  is a triple $(V(G),E(G),a)$ with a graph $(V(G),E(G))$ and (vertex-)attribute function $a \colon V(G) \to \Rb^{1 \times d}$, for some $d > 0$. Then $a(v)$ are an \new{node attributes} or \new{features} of $v$, for $v$ in $V(G)$. Equivalently, we define an $n$-vertex attributed graph $G \coloneqq (V(G),E(G),a)$ as a pair $\mG=(G,\vc{L})$, where $G = (V(G),E(G))$ and $\vc{L}$ in $\Rb^{n\times d}$ is a \new{node attribute matrix}. Here, we identify $V(G)$ with $[n]$. For a matrix $\vc{L}$ in $\Rb^{n\times d}$ and $v$ in $[n]$, we denote by $\vc{L}$ in $\Rb^{1\times d}$ the $v$th row of $\vc{L}$ such that $\vc{L}_{v} \coloneqq a(v)$. We also write $\Rb^d$ for $\Rb^{d \times 1}$. The \new{neighborhood} of $v$ in $V(G)$ is denoted by $N(v) \coloneqq  \{ u \in V(G) \mid (v, u) \in E(G) \}$.

\section{Missing proofs}

In this section, we prove \Cref{thm:main} and \Cref{thm:side}. To do so, first notice the following. In Step \ref{eq:ipm-theory:deltaw} of \Cref{alg:ipm-theory} and Step \ref{eq:ipm-practice:deltaw} of \Cref{alg:ipm-practice}, we must solve the system \Cref{eq:ipm-deltaw}. This can be done as follows. Since $(\vc{x},\vc{w},\vc{s},\vc{r})\geq0$ at any step of the algorithm, the matrix $\vc{Q}=\vc{A}\vc{D}(\vc{s})^{-1}\vc{D}(\vc{x})\vc{A}\tran+\vc{D}(\vc{w})^{-1}\vc{D}(\vc{r})$ is always symmetric and positive definite. Therefore, we can solve the system with the conjugate gradient algorithm \citep[Algorithm 5.2]{nocedal2006numerical}, say with initial point set at $\Delta\vc{w}_0=\vc{0}$ for simplicity. A specialization of the algorithm to the problem of solving \Cref{eq:ipm-deltaw} is described as Algorithm \ref{alg:cg}.

\begin{algorithm}[t]
\caption{Conjugate gradient algorithm for IPMs}
\label{alg:cg}
\begin{algorithmic}[1]
\REQUIRE An instance $I=(\vc{A}, \vc{b}, \vc{c})$ and a point $(\vc{x}, \vc{w}, \vc{s}, \vc{r})>0$.
\STATE $\vc{p}\;\leftarrow\;\vc{b}-\vc{Ax}+\sigma\mu\vc{D}(\vc{w})^{-1}\vc{1} +\vc{A}\vc{D}(\vc{s})^{-1}\vc{D}(\vc{x})[\vc{c}-\vc{A}\tran \vc{w}-\sigma\mu\vc{D}(\vc{x})^{-1}\vc{1}]$
\label{eq:cg-p0}
\STATE $\vc{v}\leftarrow -\vc{p}$
\label{eq:cg-v0}
\STATE $\Delta\vc{w}\leftarrow \vc{0}$
\label{eq:cg-w0}
\FOR{$m$ iterations}
\STATE $\vc{u} \leftarrow  [\vc{A}\vc{D}(\vc{s})^{-1}\vc{D}(\vc{x})\vc{A}\tran+\vc{D}(\vc{w})^{-1}\vc{D}(\vc{r})]\vc{p}$
\label{eq:cg-u}
\STATE $\alpha\leftarrow\vc{v}\tran\vc{v}/\vc{p}\tran \vc{u}$
\label{eq:cg-alpha}
\STATE $\Delta \vc{w} \mathrel{+}=\alpha \vc{p}$
\label{eq:cg-deltaw}
\STATE $\vc{v}_\text{new}\leftarrow \vc{v}+\alpha \vc{u}$
\label{eq:cg-vnew}
\STATE $\beta\leftarrow\vc{v}_\text{new}\tran\vc{v}_\text{new}/\vc{v}\tran\vc{v}$
\label{eq:cg-beta}
\STATE $\vc{v}\leftarrow\vc{v}_\text{new}$
\label{eq:cg-v}
\STATE $\vc{p}\leftarrow -\vc{v}+\beta \vc{p}$
\label{eq:cg-p}
\ENDFOR
\RETURN A direction $\Delta \vc{w}$ that solves the system \eqref{eq:ipm-deltaw}.
\end{algorithmic}
\end{algorithm}

\begin{lemma}\label{thm:cg-is-gnn}
There exists a MPNN $f_{\textsf{MPNN,CG}}$ composed of a $\cO(m)$ successive message-passing steps that reproduces \Cref{alg:cg}, in the sense which for any LP instance $I = (\vc{A},\vc{b},\vc{c})$ and any point $(\vc{x}, \vc{s}, \vc{w}, \vc{r})> \vc{0}$, $f_{\textsf{MPNN,CG}}$ maps the graph $G(I)$ carrying $[\vc{x}, \vc{s}]$ on the variable nodes, $[\vc{w}, \vc{r}]$ on the constraint nodes and $[\mu]$ on the objective node to the same graph $G(I)$ carrying the output $[\Delta \vc{w}]$ of \Cref{alg:cg} on the constraint nodes.
\end{lemma}

\begin{proof}
We will go through every step of the algorithm and show that it can be computed by message-passing steps on $G(I)$.
\begin{itemize}
\item For Step \ref{eq:cg-p0}, the computation can be broken down as follows. First, we can compute $\vc{h}_1\leftarrow\vc{A}\tran\vc{w}$, $\vc{h}_2\leftarrow \mu\vc{1}_n$ and $\vc{h}_3\leftarrow \vc{c}$ by a constraints-to-variables and two objective-to-variables message-passing steps, respectively. Next, one can compute $\vc{h}_4 = -\vc{x} + \vc{D}(\vc{s})^{-1}\vc{D}(\vc{x})[\vc{h}_3-\vc{h}_1-\sigma \vc{D}(\vc{x})^{-1}\vc{h}_2]$ as a local operation on variable nodes. Then we can compute $\vc{h}_5\leftarrow \vc{A}\vc{h}_4$, $\vc{h}_6\leftarrow \mu\vc{1}_m$ and $\vc{h}_7\leftarrow \vc{b}$ as a variables-to-constraints and two objective-to-constraints message-passing steps, respectively. Finally, we can compute $\vc{p}\leftarrow \vc{h}_7+\vc{h}_5+\sigma\vc{D}(\vc{w})^{-1}\vc{h}_6$ as a local operation on constraint nodes.

\item Steps \ref{eq:cg-v0} and\ref{eq:cg-w0} are just local operations on constraint nodes.

\item Step \ref{eq:cg-u} can be broken down as $\vc{h}_1\leftarrow\vc{A}\tran\vc{p}$, $\vc{h}_2\leftarrow\vc{D}(\vc{s})^{-1}\vc{D}(\vc{x})\vc{h}_1$, $\vc{h}_3\leftarrow \vc{A}\vc{h}_2$, $\vc{u}\leftarrow \vc{h}_3+\vc{D}(\vc{w})^{-1}\vc{D}(\vc{r})\vc{p}$. This can be realized as a constraints-to-variables message-passing step, a local operation on variable nodes, a variables-to-constraints message-passing step, and a local operation on constraint nodes.

\item Step \ref{eq:cg-alpha} can be broken down as $\vc{h}_1\leftarrow \vc{v}\tran \vc{v}$, $\vc{h}_2\leftarrow \vc{p}\tran \vc{u}$, $\alpha\leftarrow \vc{h}_1/\vc{h}_2$. This can be realized as a constraints-to-objective message-passing step, another constraints-to-objective message-passing step, and a local operation on the objective node.

\item Step \ref{eq:cg-deltaw} can be broken down as a message-passing step from the objective node to the constraint nodes $\vc{h}_1\leftarrow \alpha\vc{1}$, followed by a local operation on the constraint nodes $\Delta\vc{w}\leftarrow \Delta\vc{w}+\vc{D}(\vc{h}_1)\vc{p}$. 

\item Similarly, step \ref{eq:cg-vnew} can be written an objective-to-constraints message-passing step $\vc{h}_1\leftarrow \alpha\vc{1}$, followed by a local operation on constraint nodes $\Delta\vc{v}_{\text{new}}\leftarrow \vc{v}+\vc{D}(\vc{h}_1)\vc{u}$.

\item Step \ref{eq:cg-beta} can be broken down as $h_1\leftarrow \vc{v}_{\text{new}}\tran \vc{v}_{\text{new}}$, $\vc{h}_2\leftarrow \vc{v}\tran \vc{v}$, $\beta\leftarrow \vc{h}_1/\vc{h}_2$. This can be realized as a constraints-to-objective message-passing step, another constraints-to-objective message-passing step, and a local operation on the objective node.

\item Step \ref{eq:cg-v} is a local operation on constraint nodes.

\item Finally, step \ref{eq:cg-p} can be written an objective-to-constraints message-passing step $\vc{h}_1\leftarrow \beta\vc{1}$, followed by a local operation on constraint nodes $\Delta\vc{p}\leftarrow -\vc{v}+\vc{D}(\vc{h}_1)\vc{p}$.
\end{itemize}
Counting the number of successive message-passing steps, we find that Steps \ref{eq:cg-p0}--\ref{eq:cg-w0} can be realized in 8 message-passing steps, while each iteration, comprised of Steps \ref{eq:cg-u}--\ref{eq:cg-p}, can be realized in 9 message-passing steps, completing the proof.
\end{proof}

We now move on with the proofs of \Cref{thm:main} and \Cref{thm:side}.

\begin{theorem}\label{thm:main_app}
There exists an MPNN $f_{\mathsf{MPNN}, \mathsf{IPM1}}$ composed of $\cO(m)$ message-passing steps that reproduces an iteration of \Cref{alg:ipm-theory}, in the sense that for any LP instance $I=(\vc{A},\vc{b},\vc{c})$ and any iteration step $t\geq0$, $f_{\mathsf{MPNN}, \mathsf{IPM1}}$ maps the graph $G(I)$ carrying $[\vc{x}_t, \vc{s}_t]$ on the variable nodes and $[\vc{w}_t, \vc{r}_t]$ on the constraint nodes to the same graph $G(I)$ carrying  $[\vc{x}_{t+1}, \vc{s}_{t+1}]$ on the variable nodes and $[\vc{w}_{t+1}, \vc{r}_{t+1}]$ on the constraint nodes.
\end{theorem}
\begin{proof}
We need to check that every step can be computed by message-passing steps over $G(I)$.
\begin{itemize}
\item Step \ref{eq:ipm-theory:mu} can be written as $\vc{h}_1\leftarrow \vc{x}\tran\vc{s}$, $\vc{h}_2\leftarrow \vc{w}\tran\vc{r}$, $\mu = (\vc{h}_1+\vc{h}_2)/(n+m)$. These can be realized as a variable-to-objective message-passing step, a constraints-to-objective message-passing step, and a local operation on the objective node, respectively.
\item Step \ref{eq:ipm-theory:deltaw} can be written as message-passing steps by \Cref{thm:cg-is-gnn}.

\item Step \ref{eq:ipm-theory:deltax} can be broken down as follows. We can compute $\vc{h}_1\leftarrow\vc{A}\tran[\vc{w}+\Delta\vc{w}]$, $\vc{h}_2\leftarrow \mu\vc{1}_n$ and $\vc{h}_3\leftarrow \vc{c}$ by a constraints-to-variables and two objective-to-variables message-passing steps, respectively. Then, one can compute $\Delta\vc{x} \leftarrow \vc{D}(\vc{s})^{-1}\vc{D}(\vc{x})\vc{h}_1-\vc{h}_3+\sigma \vc{D}(\vc{x})^{-1}\vc{h}_2$ by a local operation on variable nodes.

\item Step \ref{eq:ipm-theory:deltas} can be realized by taking an objective-to-variables message-passing step $\vc{h}_1\leftarrow \mu\vc{1}_n$, and computing $\Delta \vc{s}\leftarrow \sigma\vc{D}(\vc{x})^{-1}\vc{h}_1-\vc{s}-\vc{D}(\vc{x})^{-1}\vc{D}(\vc{s})\Delta\vc{x}$.

\item Step  \ref{eq:ipm-theory:deltar} can be realized by taking an objective-to-constraints message-passing step $\vc{h}_1\leftarrow \mu\vc{1}_m$, and computing $\Delta \vc{r}\leftarrow \sigma\vc{D}(\vc{w})^{-1}\vc{h}_1-\vc{r}-\vc{D}(\vc{w})^{-1}\vc{D}(\vc{r})\Delta\vc{w}$.

 \item Step \ref{eq:ipm-theory:alpha} can be performed by message-passing steps as follows. We need to find the largest $\alpha \in (0,1)$ such that
\begin{align}
&(\vc{x}+\alpha\Delta \vc{x})_i(\vc{s}+\alpha\Delta \vc{s})_i \\
&\geq \gamma\frac{(\vc{x}+\alpha\Delta \vc{x})\tran (\vc{s}+\alpha\Delta \vc{s})+(\vc{w}+\alpha\Delta \vc{w})\tran (\vc{r}+\alpha\Delta \vc{r})}{n+m}
\label{eq:ipm-theory:alpha-i}
\end{align}
for every $i\in[n]$ and
\begin{align}
&(\vc{w}+\alpha\Delta \vc{w})_j(\vc{r}+\alpha\Delta \vc{r})_j\} \\
&\geq \gamma\frac{(\vc{x}+\alpha\Delta \vc{x})\tran (\vc{s}+\alpha\Delta \vc{s})+(\vc{w}+\alpha\Delta \vc{w})\tran (\vc{r}+\alpha\Delta \vc{r})}{n+m}
\label{eq:ipm-theory:alpha-j}
\end{align}
for every $j\in[m]$. Equivalently, for each $i\in[n]$, we can find the largest $\alpha_i<1$ such that \ref{eq:ipm-theory:alpha-i} holds, that is such that
\begin{align*}
&\alpha_i^2\left(\Delta x_i\Delta s_i-\gamma\frac{\Delta\vc{x}\tran\Delta\vc{s}+\Delta\vc{w}\tran\Delta\vc{r}}{n+m}\right) \\
&+\alpha_i\left(x_i\Delta s_i+\Delta x_i s_i-\gamma\frac{\vc{x}\tran\Delta\vc{s}+\Delta\vc{x}\tran\vc{s}
+\vc{w}\tran\Delta\vc{r}+\Delta\vc{w}\tran\vc{r}}{n+m}\right) \\
&+\left(x_is_i-\gamma\frac{\vc{x}\tran\vc{s}+\vc{w}\tran\vc{r}}{n+m}\right)
\geq 0
\end{align*}
holds; and similarly, find the largest $\bar{\alpha}_j<1$ such that \ref{eq:ipm-theory:alpha-j} holds, that is such that
\begin{align*}
&\bar{\alpha}_j^2\left(\Delta w_j\Delta r_j-\gamma\frac{\Delta\vc{x}\tran\Delta\vc{s}+\Delta\vc{w}\tran\Delta\vc{r}}{n+m}\right) \\
&+\bar{\alpha}_j\left(w_j\Delta r_j+\Delta w_j r_j-\gamma\frac{\vc{x}\tran\Delta\vc{s}+\Delta\vc{x}\tran\vc{s}
+\vc{w}\tran\Delta\vc{r}+\Delta\vc{w}\tran\vc{r}}{n+m}\right) \\
&+\left(w_jr_j-\gamma\frac{\vc{x}\tran\vc{s}+\vc{w}\tran\vc{r}}{n+m}\right)
\geq 0
\end{align*}
holds; then $\alpha=\min_{i,j}\{\alpha_i, \bar{\alpha}_j\}$. 

This can be computed by message-passing steps as follows. First, we can compute $\vc{h}_1\leftarrow \Delta\vc{x}\tran\Delta\vc{s}$, $\vc{h}_2\leftarrow \Delta\vc{x}\tran\vc{s}$, $\vc{h}_3\leftarrow \vc{x}\tran\Delta\vc{s}$ and $\vc{h}_4\leftarrow \vc{x}\tran\vc{s}$ by variable-to-objective message-passing steps; and similarly $\bar{\vc{h}}_1\leftarrow \Delta\vc{w}\tran\Delta\vc{r}$, $\bar{\vc{h}}_2\leftarrow \Delta\vc{w}\tran\vc{r}$, $\bar{\vc{h}}_3\leftarrow \vc{w}\tran\Delta\vc{r}$, $\bar{\vc{h}}_4\leftarrow \vc{w}\tran\vc{r}$ by constraints-to-objective message-passing steps. The quantities $t_1\leftarrow \gamma(\vc{h}_1+\bar{\vc{h}}_1)/(n+m)$, $t_2\leftarrow \gamma(\vc{h}_2+\vc{h}_3+\bar{\vc{h}}_2+\bar{\vc{h}_3})/(n+m)$ and $t_3\leftarrow \gamma(\vc{h}_4+\bar{\vc{h}}_4)/(n+m)$ can then be computed by local operations on the objective node, and returned to the variable and constraint nodes by objective-to-variables message-passing steps $\vc{t}_1\leftarrow t_1\vc{1}_n$, $\vc{t}_2\leftarrow t_2\vc{1}_n$, $\vc{t}_3\leftarrow t_3\vc{1}_n$ and objective-to-constraint message-passing steps $\vc{\bar{t}}_1\leftarrow t_1\vc{1}_n$, $\vc{\bar{t}}_2\leftarrow t_2\vc{1}_n$, $\vc{\bar{t}}_3\leftarrow t_3\vc{1}_n$. Then, on each variable node $v_i$, we can solve
\begin{align*}
\alpha_i = &\max \{\alpha\in(0,1)\;\vert\;\alpha^2(\Delta x_i\Delta s_i-t_{1i}) \\
&+\alpha(x_i\Delta s_i+\Delta x_i s_i-t_{2i})+(x_is_i-t_{3i})\geq 0\}
\end{align*}
as a local operation, and similarly, on each constraint node $c_j$, we can find
\begin{align*}
\bar{\alpha}_j = &\max \{\alpha\in(0,1)\;\vert\;\alpha^2(\Delta w_i\Delta r_i-\bar{t}_{1i}) \\
&+\alpha(w_i\Delta r_i+\Delta w_i r_i-\bar{t}_{2i})+(w_ir_i-\bar{t}_{3i})\geq 0\}
\end{align*}
as a local operation. Finally, we can compute $\alpha_v\leftarrow \min_i\alpha_i$ as a variables-to-objective message-passing step, and $\alpha_c\leftarrow \min_j\bar{\alpha}_j$ as a constraints-to-objective message-passing step, and finally take $\alpha=\min(\alpha_v, \alpha_c)$ as a local operation on the objective node.

\item Finally, step \ref{eq:ipm-theory:update} can be performed by taking objective-to-variables, objective-to-constraints message-passing steps $\vc{h}_1\leftarrow \alpha\vc{1}_n$, $\vc{\bar{h}}_1\leftarrow \alpha\vc{1}_m$, and taking local operations $(\vc{x},\vc{s})\leftarrow (\vc{x}+\vc{D}(\vc{h}_1)\Delta\vc{x}, \vc{s}+\vc{D}(\vc{h}_1)\Delta\vc{s})$ and $(\vc{w},\vc{r})\leftarrow (\vc{w}+\vc{D}(\vc{\bar{h}}_1)\Delta\vc{w}, \vc{r}+\vc{D}(\vc{\bar{h}}_1)\Delta\vc{r})$ on variable and constraint nodes respectively.
\end{itemize}

Counting the number of successive message-passing steps, we find that all steps can be realized in 23 message-passing steps, plus the $\cO(m)$ steps of Step \ref{eq:ipm-theory:deltaw}, completing the proof.
\end{proof}

Further, we show an analogous result for~\cref{alg:ipm-practice}. 

\begin{proposition}\label{thm:side_app}
There exists an MPNN $f_{\mathsf{MPNN}, \mathsf{IPM2}}$ composed of $\cO(m)$ message-passing steps that reproduces each iteration of \Cref{alg:ipm-practice}, in the sense that for any LP instance $I=(\vc{A},\vc{b},\vc{c})$ and any iteration step $t\geq0$, $f_{\mathsf{MPNN}, \mathsf{IPM2}}$ maps the graph $G(I)$ carrying $[\vc{x}_t, \vc{s}_t]$ on the variable nodes, $[\vc{w}_t, \vc{r}_t]$ on the constraint nodes and $[\mu_t]$ on the objective node to the same graph $G(I)$ carrying  $[\vc{x}_{t+1}, \vc{s}_{t+1}]$ on the variable nodes, $[\vc{w}_{t+1}, \vc{r}_{t+1}]$ on the constraint nodes and $[\mu_{t+1}]$ on the objective node.
\end{proposition}

\begin{proof}
We need to show that every step can be written as message-passing steps over $G(I)$, as in Theorem \ref{thm:main}. Steps \ref{eq:ipm-practice:deltaw}-\ref{eq:ipm-practice:deltar} are the same as \ref{eq:ipm-theory:deltaw}-\ref{eq:ipm-theory:deltar}, so by Theorem \ref{thm:main}, they can be written as message-passing steps. 
This leaves to check steps \ref{eq:ipm-practice:alpha}-\ref{eq:ipm-practice:mu}.

\begin{itemize}
\item The analysis of Step \ref{eq:ipm-practice:alpha} is similar to the analysis of Step \ref{eq:ipm-theory:alpha} of Algorithm \ref{alg:ipm-theory} in Theorem \ref{thm:main}, but simpler. On every variable node, we can compute
\begin{align*}
\alpha_i = &\max \{\alpha\in(0,\infty)\;\vert\;\alpha^2\Delta x_i\Delta s_i \\
&+\alpha(x_i\Delta s_i+\Delta x_i s_i)+x_is_i\geq 0\}
\end{align*}
as a local operation, and similarly, on each constraint node $c_j$ we can find
\begin{align*}
\bar{\alpha}_j = &\max \{\alpha\in(0,\infty)\;\vert\;\alpha^2\Delta w_i\Delta r_i \\
&+\alpha(w_i\Delta r_i+\Delta w_i r_i)+w_ir_i\geq 0\}
\end{align*}
as a local operation. Then we can compute $\alpha_v\leftarrow \min_i\alpha_i$ as a variables-to-objective message-passing step, and $\alpha_c\leftarrow \min_j\bar{\alpha}_j$ as a constraints-to-objective message-passing step, and finally take $\alpha=\min(\alpha_v, \alpha_c)$ as a local operation on the objective node.
\item The analysis of step \ref{eq:ipm-practice:update} is similar to the analysis of Step \ref{eq:ipm-theory:update} of Algorithm \ref{alg:ipm-theory} in Theorem \ref{thm:main} as well. It can be performed by taking objective-to-variables, objective-to-constraints message-passing steps $\vc{h}_1\leftarrow \alpha\vc{1}_n$, $\vc{\bar{h}}_1\leftarrow \alpha\vc{1}_m$, and taking local operations $(\vc{x},\vc{s})\leftarrow (\vc{x}+0.99\vc{D}(\vc{h}_1)\Delta\vc{x}, \vc{s}+0.99\vc{D}(\vc{h}_1)\Delta\vc{s})$ and $(\vc{w},\vc{r})\leftarrow (\vc{w}+0.99\vc{D}(\vc{\bar{h}}_1)\Delta\vc{w}, \vc{r}+0.99\vc{D}(\vc{\bar{h}}_1)\Delta\vc{r})$ on variable and constraint nodes respectively.
\item Finally, step \ref{eq:ipm-practice:mu} is a local operation on the objective node.
\end{itemize}
Just like in \Cref{thm:main_app}, we can see all the computations can be realized in $\cO(m)$ message-passing steps, completing the proof.
\end{proof}

\section{Details of IPM-MPNNs}
\label{sec:conv_form}
In the following, we outline details with regard to the specific MPNN layers used in~\cref{sec:experiments}. We follow the notation outlined in~\cref{sec:ipm-mpnn}. Furthermore, let $\textsf{MLP}$ be a multi-layer perceptron, whose subscript denotes its role. Specifically, $\textsf{MLP}_{*}$ is for node initialization or node updating after gathering message function, $\textsf{MLP}_{* \rightarrow *}$ is for message vector mapping, and $\textsf{MLP}_{**}$ is for edge feature embedding in each layer. At the initialization $t=0$, we obtain node embeddings by
\begin{align*}
    \vc{h}_v^{(0)} &\coloneqq \textsf{MLP}_{\text{v}}^{(0)} \left( \vc{x}_v \right), \forall v \in V(I)\\
    \vc{h}_c^{(0)} &\coloneqq \textsf{MLP}_{\text{c}}^{(0)} \left( \vc{x}_c \right), \forall c \in C(I)\\
    \vc{h}_o^{(0)} &\coloneqq \textsf{MLP}_{\text{o}}^{(0)} \left( \vc{x}_o \right).
\end{align*}

Then, a GCN layer updates the constraint, objective, and variable nodes as follows:
{\allowdisplaybreaks
\begin{align*}
    \vc{h}_c^{(t)} \coloneqq \textsf{MLP}^{(t)}_{\text{c}}\biggl[ &\textsf{MLP}^{(t)}_{\text{c} \rightarrow \text{c}}\left(\vc{h}_c^{(t-1)} \right) + \\
     &\textsf{MLP}^{(t)}_{\text{o} \rightarrow \text{c}}\left(\dfrac{1}{\sqrt{d_o d_c}} \left( \vc{h}_o^{(t-1)} + \textsf{MLP}_{\text{oc}}^{(t)}\left(\vc{e}_{oc}\right) \right)\right) + \\
     &\textsf{MLP}^{(t)}_{\text{v} \rightarrow \text{c}}\left( \sum_{v \in {N}_c \cap V(I)} \dfrac{1}{\sqrt{d_v d_c}} \left(\vc{h}_v^{(t-1)} + \textsf{MLP}_{\text{vc}}^{(t)}\left(\vc{e}_{vc}\right)\right)   \right) \biggr] \\
    \vc{h}_o^{(t)} \coloneqq \textsf{MLP}^{(t)}_{\text{o}}\biggl[ &\textsf{MLP}^{(t)}_{\text{o} \rightarrow \text{o}}\left(\vc{h}_o^{(t-1)} \right) + \\
     &\textsf{MLP}^{(t)}_{\text{c} \rightarrow \text{o}}\left( \sum_{c \in C(I)} \dfrac{1}{\sqrt{d_o d_c}} \left( \vc{h}_c^{(t)} +\textsf{MLP}_{\text{co}}^{(t)}\left(\vc{e}_{co}\right) \right)\right) + \\
     &\textsf{MLP}^{(t)}_{\text{v} \rightarrow \text{o}}\left( \sum_{v \in V(I)} \dfrac{1}{\sqrt{d_o d_v}} \left( \vc{h}_v^{(t-1)} + \textsf{MLP}_{\text{vo}}^{(t)}\left(\vc{e}_{vo}\right)  \right) \right) \biggr]\\
     \vc{h}_v^{(t)} \coloneqq \textsf{MLP}^{(t)}_{\text{v}}\biggl[ &\textsf{MLP}^{(t)}_{\text{v} \rightarrow \text{v}}\left(\vc{h}_v^{(t-1)} \right) + \\
     &\textsf{MLP}^{(t)}_{\text{o} \rightarrow \text{v}}\left(\dfrac{1}{\sqrt{d_o d_c}} \left( \vc{h}_o^{(t)} + \textsf{MLP}_{\text{ov}}^{(t)}\left(\vc{e}_{ov}\right) \right) \right)+ \\
     &\textsf{MLP}^{(t)}_{\text{c} \rightarrow \text{v}}\left( \sum_{c \in {N}_v \cap C(I)} \dfrac{1}{\sqrt{d_c d_v}} \left(\vc{h}_c^{(t)} + \textsf{MLP}_{\text{cv}}^{(t)}\left(\vc{e}_{cv}\right)  \right) \right) \biggr].
\end{align*}}

Similarly, for the  GIN layer, we have:

{\allowdisplaybreaks
\begin{align*}
    \vc{h}_c^{(t)} \coloneqq \textsf{MLP}^{(t)}_{\text{c}}\biggl[&\left(1 + \epsilon_{c}^{(t)} \right) \textsf{MLP}^{(t)}_{\text{c} \rightarrow \text{c}}\left(\vc{h}_c^{(t-1)} \right) + \\
     &\textsf{MLP}^{(t)}_{\text{o} \rightarrow \text{c}}\left(\vc{h}_o^{(t-1)} + \textsf{MLP}_{\text{oc}}^{(t)}\left(\vc{e}_{oc}\right) \right) + \\
     &\textsf{MLP}^{(t)}_{\text{v} \rightarrow \text{c}}\left( \sum_{v \in {N}_c \cap V(I)} \left( \vc{h}_v^{(t-1)} + \textsf{MLP}_{\text{vc}}^{(t)}\left(\vc{e}_{vc}\right) \right) \right) \biggr] \\
    \vc{h}_o^{(t)} \coloneqq \textsf{MLP}^{(t)}_{\text{o}}\biggl[&\left(1 + \epsilon_{o}^{(t)} \right) \textsf{MLP}^{(t)}_{o \rightarrow o}\left(\vc{h}_o^{(t-1)} \right) + \\
     &\textsf{MLP}^{(t)}_{\text{c} \rightarrow \text{o}}\left( \sum_{c \in C(I)} \left( \vc{h}_c^{(t)} + \textsf{MLP}_{\text{co}}^{(t)}\left(\vc{e}_{co}\right) \right)\right) + \\
     &\textsf{MLP}^{(t)}_{\text{v} \rightarrow \text{o}}\left( \sum_{v \in V(I)} \left(\vc{h}_v^{(t-1)}+\textsf{MLP}_{\text{vo}}^{(t)}\left(\vc{e}_{vo}\right) \right) \right) \biggr]\\
     \vc{h}_v^{(t)} \coloneqq \textsf{MLP}^{(t)}_{\text{v}}\biggl[&\left(1 + \epsilon_{v}^{(t)} \right) \textsf{MLP}^{(t)}_{\text{v} \rightarrow \text{v}}\left(\vc{h}_v^{(t-1)} \right) + \\
     &\textsf{MLP}^{(t)}_{\text{o} \rightarrow \text{v}}\left(\vc{h}_o^{(t)} +\textsf{MLP}_{\text{ov}}^{(t)}\left(\vc{e}_{ov}\right) \right) + \\
     &\textsf{MLP}^{(t)}_{\text{c} \rightarrow \text{v}}\left( \sum_{c \in {N}_v \cap C(I)} \left( \vc{h}_c^{(t)} + \textsf{MLP}_{\text{cv}}^{(t)}\left(\vc{e}_{cv}\right) \right) \right) \biggr].
\end{align*}}

For the  GEN layer, we have:
{\allowdisplaybreaks
\begin{align*}
    \vc{h}_c^{(t)} \coloneqq \textsf{MLP}^{(t)}_{\text{c}}\biggl[&\textsf{MLP}^{(t)}_{\text{c} \rightarrow \text{c}}\left(\vc{h}_c^{(t-1)} \right) + \\
     &\textsf{MLP}^{(t)}_{\text{o} \rightarrow \text{c}}\left(\vc{h}_o^{(t-1)} + \textsf{MLP}_{\text{oc}}^{(t)}\left(\vc{e}_{oc}\right) + \epsilon_{\text{o} \rightarrow \text{c}}^{(t)} \right) + \\
     &\textsf{MLP}^{(t)}_{\text{v} \rightarrow \text{c}} \left( \textsf{MSG} \left( \left\{\!\!\left\{ \vc{h}_v^{(t-1)} + \textsf{MLP}_{\text{vc}}^{(t)}\left(\vc{e}_{vc}\right) + \epsilon_{\text{v} \rightarrow \text{c}}^{(t)} \mid v \in {N}_c \cap V(I) \right\}\!\!\right\} \right) \right) \biggr] \\
    \vc{h}_o^{(t)} \coloneqq \textsf{MLP}^{(t)}_{\text{o}}\biggl[&\textsf{MLP}^{(t)}_{\text{o} \rightarrow \text{o}}\left(\vc{h}_o^{(t-1)} \right) + \\
     &\textsf{MLP}^{(t)}_{\text{c} \rightarrow \text{o}}\left( \textsf{MSG} \left( \left\{\!\!\left\{ \vc{h}_c^{(t-1)}+ \textsf{MLP}_{\text{co}}^{(t)}\left(\vc{e}_{co}\right) + \epsilon_{\text{c} \rightarrow \text{o}}^{(t)} \mid c \in C(I) \right\}\!\!\right\} \right) \right) + \\
     &\textsf{MLP}^{(t)}_{\text{v} \rightarrow \text{o}}\left( \textsf{MSG} \left( \left\{\!\!\left\{ \vc{h}_c^{(t-1)} + \textsf{MLP}_{\text{vc}}^{(t)}\left(\vc{e}_{vc}\right) + \epsilon_{\text{v} \rightarrow \text{o}}^{(t)} \mid v \in V(I) \right\}\!\!\right\} \right) \right) \biggr]\\
     \vc{h}_v^{(t)} \coloneqq \textsf{MLP}^{(t)}_{\text{v}}\biggl[&\textsf{MLP}^{(t)}_{\text{v} \rightarrow \text{v}}\left(\vc{h}_v^{(t-1)} \right) + \\
     &\textsf{MLP}^{(t)}_{\text{o} \rightarrow \text{v}}\left(\vc{h}_o^{(t)} + \textsf{MLP}_{\text{ov}}^{(t)}\left(\vc{e}_{ov}\right) + \epsilon_{\text{o} \rightarrow \text{v}}^{(t)} \right) + \\
     &\textsf{MLP}^{(t)}_{\text{c} \rightarrow \text{v}}\left( \textsf{MSG} \left( \left\{\!\!\left\{ \vc{h}_c^{(t-1)}+ \textsf{MLP}_{\text{cv}}^{(t)}\left(\vc{e}_{cv}\right) + \epsilon_{\text{c} \rightarrow \text{v}}^{(t)} \mid c \in {N}_v \cap C(I) \right\}\!\!\right\} \right) \right) \biggr],
\end{align*}}
where  $\textsf{MSG}$ is the softmax aggregation with $\tau=1$, i.e.,
\begin{equation*}
    \textsf{softmax}(\mathcal{X} \mid \tau) = \sum_{\vc{x}_i\in\mathcal{X}}
        \frac{\exp(\tau \cdot\vc{x}_i)}{\sum_{\vc{x}_j\in\mathcal{X}}
        \exp(\tau \cdot\vc{x}_j)}\cdot\vc{x}_{i}.
\end{equation*}

\section{Details of datasets}\label{app:dataset}

In the following, we describe our datasets. 

\subsection{Combinatorial optimization problems}

In the following, we briefly the combinatorial optimization problems. 

\xhdr{Set cover problem} The set cover problem aims cover the universe $U$ with a collection of given subsets $S = \{S_1, S_2, \ldots, S_m\}$ satisfying $\cup_{i=1}^m S_i = U$, with the target of minimizing the objective function. Formally, let $x_i$ be the variable deciding whether subset $S_i$ is selected, and $c_i$ the cost per subset, we have:
\begin{equation*}
\begin{split}
    \min_{\vc{x}}   &\sum_{i=1}^m c_i x_i \\
\text{s.t.} &\sum_{i:u \in S_i} x_i \geq 1, \forall u \in U \\
    &x_i \in \{0, 1\}, \forall i \in [m].
\end{split}
\end{equation*}

\xhdr{Maximal independent set} Given an undirected graph $G$ with node set $V(G)$ and edge set $E(G)$, the goal of the maximal independent set problem is to find a set of nodes where no pairs of them are connected. If we use $x_i$ to denote a node $i$ is selected or not, we have:

\begin{equation*}
\begin{split}
    \max_{\vc{x}}  &\sum_{i \in V} x_i \\
\text{s.t. } & x_u + x_v \leq 1, \forall (u, v) \in E(G), u, v \in V(G) \\
    &x_i \in \{0, 1\}, \forall i \in V(G).
\end{split}
\end{equation*}

\xhdr{Combinatorial auction} Suppose there are a set of items $M$ and one of bidders $N$. Each bidder $i \in N$ maintains a set of bids $B_i$, each bid $b \in B_i$ is associated with a subset $S_{ib} \subseteq M$ of items and a value $v_{ib}$ that the bidder $i$ is willing to pay for this subset. The binary decision variable $x_{ib}$ is $1$ if the bid $b$ by bidder $i$ is accepted or $0$ otherwise. The MILP formulation of the problem is as follows:

\begin{equation}
\begin{split}
    \max_{\vc{x}} & \sum_{i \in N} \sum_{b \in B_i} v_{ib} x_{ib} \\
   \text{s.t.} & \sum_{i \in N} \sum_{b \in B_i : j \in S_{ib}} x_{ib} \leq 1, \forall j \in M, \\
        & x_{ib} \in \{0,1\}, \forall i \in N, b \in B_i.
\end{split}
\end{equation}

\xhdr{Capacitated facility location} Given a set of facilities $M$ and another of customers $N$, we aim to build facilities and satisfy the demand of the customers at minimum cost. Let $y_j, j \in M$ be the binary decision of building the facility $j$ or not, and $x_{ij}$ be a continuous variable indicating the fraction of demand facility $j$ sends to customer $i \in N$. Let $d_i \in \mathbb{R}^+$ be the amount of demand of customer $i$, and $v_j$ be the volume of the facility $j$, $c_{ij}$ the cost of shipment and $f_j$ the cost of building facility $j$, we formulate the problem as follows:

\begin{equation}
\begin{split}
    \min_{\vc{x}}  &\sum_{j \in M} f_j y_j + \sum_{i \in N} \sum_{j \in M} c_{ij} x_{ij} \\
    \text{s.t.} & \sum_{j \in M} x_{ij} = 1, \forall i  \in N \\
    & \sum_{i \in N} d_i x_{ij} \leq v_j y_j, \forall i  \in N, j \in M \\
    &y_j \in \{0, 1\}, x_{ij} \in [0, 1], \forall i  \in N, j \in M.
\end{split}
\end{equation}

\subsection{Generation of instances} 

We propose various sizes of generated instances; see~\cref{tab: setc_size,tab: indset_size,tab: cauc_size,tab: fac_size} for the size parameters of each dataset. The generation of instances follows the setting of \citet{Gas+2019}. For the set covering instances, problems are generated with 15-20 rows and columns and a constraint matrix density of 0.15 for mini instances. For small instances, we used 30-50 rows and 50-70 columns with a density of 0.05. Large instances have 300-500 rows and 500-700 columns with a density of 0.01. We employ the Erd\H{o}s--R'enyi random graph as the foundational graph when generating maximal independent set instances, designating 10-20 nodes for mini instances, 50-70 nodes for small instances, and 500-700 nodes for large instances. For combinatorial auction instances, we modulate the size by varying the number of items and bids: specifically, 20 items and bids are set for mini instances, 50-80 for small instances, and 500-800 for large instances. Lastly, for the capacitated facility location instances, we set 3-5 customers and facilities for mini instances, 10 for both in small instances, and 20-30 in large ones.

\begin{table*}[t]
\caption{Sizes of Setcover.}
\label{tab: setc_size}
\begin{center}
\begin{small}
\begin{tabular}{ccccc}
\toprule
Size & Num. Row & Num. Col. &  Density & Num. instances \\
\midrule
Mini & [15, 20] & [15, 20] & 0.15 & \num{1000} \\
Small & [30, 50] & [50, 70] & 0.05 & \num{10000}\\
Large & [300, 500] & [500, 700] & 0.01 & \num{10000}\\
\bottomrule
\end{tabular}
\end{small}
\end{center}
\end{table*}

\begin{table*}[t]
\caption{Sizes of maximal independent set instances.}
\label{tab: indset_size}
\begin{center}
\begin{small}
\begin{tabular}{cccc}
\toprule
Size & Num. nodes & Affinity &  Num. instances \\
\midrule
Mini & [10, 20] & 2 & \num{1000} \\
Small & [50, 70] & 2 & \num{10000}\\
Large & [500, 700] & 2 & \num{10000}\\
\bottomrule
\end{tabular}
\end{small}
\end{center}
\end{table*}

\begin{table*}[t]
\caption{Sizes of combinatorial auction instances.}
\label{tab: cauc_size}
\begin{center}
\begin{small}
\begin{tabular}{cccc}
\toprule
Size & Num. items & Num. bids &  Num. instances \\
\midrule
Mini & 20 & 20 & \num{1000} \\
Small & $[50, 80]$ & $[50, 80]$ & \num{10000}\\
Large & $[300, 500]$ & $[300, 500]$ & \num{10000}\\
\bottomrule
\end{tabular}
\end{small}
\end{center}
\end{table*}

\begin{table*}[t]
\caption{Sizes of capacitated facility location instances.}
\label{tab: fac_size}
\begin{center}
\begin{small}
\begin{tabular}{ccccc}
\toprule
Size & Num. customers & Num. facilities & Ratio &  Num. instances \\
\midrule
Mini & [3, 5] & [3, 5] & 5 & \num{1000} \\
Small & 10 & 10 & 5 & \num{10000}\\
Large & [20, 30] & [20, 30] & 5 & \num{10000}\\
\bottomrule
\end{tabular}
\end{small}
\end{center}
\end{table*}

\section{Training parameters}
\label{sec:hyperparam}

For all the experiments, we train the neural networks with Adam optimizer with default hyperparamters, and run for at most 1000 epochs. During training, we leverage learning rate decay with right to the validation objective gap metric with a decay ratio of 0.5 and patience 50. We terminate the run at patience 100 epochs. Besides, we display the other task-specific hyperparameters in \cref{tab:hyperparams}, which are the batch size, number of MPNN layers as well as the number of sampled IPM solver steps, the step decay factor introduced in \cref{eq:main_loss}, the loss weight combination in \cref{eq:full_loss}, plus the weight decay of the optimizer. 

With regard to the bipartiteness ablation study, we also tune the hyperparameters for the sake of fair comparison. The choosen hyperparameters are listed in \cref{tab:hyperparams_bipart}.  Moreover, the hyperparameter configurations for the ODE approach baseline \citet{Wu+2023} are shown in~\cref{tab:hyperparams_ode}.

\begin{table*}[t]
\caption{Training hyperparameters of our tripartite MPNN main experiments, $v^a$ represents $v \times 10^a$. }
\label{tab:hyperparams}
\begin{center}
\resizebox{\textwidth}{!}{
\begin{small}
\begin{tabular}{ccccccccccc}
\toprule
Instances & Size & MPNN & Batch size & Num. layers & Hidden dim. & $\alpha$ & $w_{\text{var}}$ & $w_{\text{obj}}$ & $w_{\text{cons}}$ & Weight decay \\
\midrule
{\multirow{6}{*}{Setcover}} & {\multirow{3}{*}{Small}} 
& GEN & 512 & 8  & 180 & 0.2 & 1.2 & 0.8 & 0.2 & $1.2^{-6}$ \\
& & GCN & 512 & 8 & 180& 0.8 & 1.0 & 0.3 & 2.2 & $4.4^{-7}$ \\
& & GIN & 512 & 8 & 180 & 0.7 & 1.0 & 2.4 & 7.5 & $5.6^{-6}$ \\

& {\multirow{3}{*}{Large}} 
& GEN & 128 & 8  & 180 & 0.2 & 1.2 & 0.8 & 0.2 & $1.2^{-6}$ \\
& & GCN & 128 & 8 & 180 & 0.2 & 1.0 & 2.2 & 0.3 & $1.5^{-8}$ \\
& & GIN & 128 & 8 & 180 & 0.7 & 1.0 & 4.5 & 2.2 & $2.8^{-8}$ \\

\midrule
{\multirow{6}{*}{Indset}} & {\multirow{3}{*}{Small}} 
& GEN & 512 & 8  & 180 & 0.2 & 1.2 & 0.8 & 0.2 & $1.2^{-6}$ \\
& & GCN & 512 & 8 & 180& 0.5 & 1.0 & 4.5 & 9.6 & $2.0^{-7}$ \\
& & GIN & 512 & 8 & 180 & 0.7 & 1.0 & 2.4 & 7.5 & $5.6^{-6}$ \\

& {\multirow{3}{*}{Large}} 
& GEN & 128 & 8  & 180 & 0.2 & 1.2 & 0.8 & 0.2 & $1.2^{-6}$ \\
& & GCN & 128 & 8 & 180& 0.5 & 1.0 & 4.5 & 9.6 & $2.0^{-7}$ \\
& & GIN & 128 & 8 & 180 & 0.7 & 1.0 & 2.4 & 7.5 & $5.6^{-6}$ \\

\midrule
{\multirow{6}{*}{Cauc.}} & {\multirow{3}{*}{Small}} 
& GEN & 512 & 8  & 180 & 0.9 & 1.0 & 4.6 & 5.3 & 0.0 \\
& & GCN & 512 & 8 & 180 & 0.4 & 1.0 & 3.4 & 5.8 & 0.0 \\
& & GIN & 512 & 8 & 180 & 0.6 & 1.0 & 4.3 & 6.3 & 0.0 \\

& {\multirow{3}{*}{Large}} 
& GEN & 128 & 8  & 180 & 0.9 & 1.0 & 4.6 & 5.3 & 0.0 \\
& & GCN & 128 & 8 & 180& 0.4 & 1.0 & 3.4 & 5.8 & 0.0 \\
& & GIN & 128 & 8 & 180 & 0.6 & 1.0 & 4.3 & 6.3 & 0.0 \\

\midrule
{\multirow{6}{*}{Fac.}} & {\multirow{3}{*}{Small}} 
& GEN & 512 & 8  & 180 & 0.8 & 1.0 & 3.0 & 8.2 & $3.8^{-6}$ \\
& & GCN & 512 & 8 & 96 & 0.6 & 1.0 & 8.7 & 9.6 & $4.5^{-7}$  \\
& & GIN & 512 & 8 & 180 & 0.8 & 1.0 & 1.3 & 4.6 & $1.0^{-7}$ \\

& {\multirow{3}{*}{Large}} 
& GEN & 128 & 8  & 180 & 0.8 & 1.0 & 3.09 & 8.2 & $3.8^{-6}$ \\
& & GCN & 128 & 8 & 96 & 0.6 & 1.0 & 8.7 & 9.6 & $4.5^{-7}$ \\
& & GIN & 128 & 8 & 180 & 0.9 & 1.0 & 2.5 & 4.0 & $1.0^{-5}$ \\

\bottomrule
\end{tabular}
\end{small}
}
\end{center}
\end{table*}

\begin{table*}[t]
\caption{Training hyperparameters of our bipartite MPNN ablation experiments, $v^a$ represents $v \times 10^a$. }
\label{tab:hyperparams_bipart}
\begin{center}
\resizebox{\textwidth}{!}{
\begin{small}
\begin{tabular}{ccccccccccc}
\toprule
Instances & Size & MPNN & Batch size & Num. layers & Hidden dim. & $\alpha$ & $w_{\text{var}}$ & $w_{\text{obj}}$ & $w_{\text{cons}}$ & Weight decay \\
\midrule
{\multirow{6}{*}{Setcover}} & {\multirow{3}{*}{Small}} 
& GEN & 512 & 8  & 32 & 0.8 & 1.0 & 2.6 & 0.8  & $1.0^{-6}$ \\
& & GCN  & 512 & 8  & 32 & 0.9 & 1.0 & 5.5  & 1.1 & $1.0^{-5}$ \\
& & GIN & 512 & 8  & 64 & 0.3 & 1.0 & 4.7  & 0.8 & $1.0^{-5}$ \\

& {\multirow{3}{*}{Large}} 
& GEN & 128 & 8  & 32 & 0.8 & 1.0 & 2.6 & 0.8  & $1.0^{-6}$ \\
& & GCN  & 128 & 8  & 32 & 0.9 & 1.0 & 5.5  & 1.1 & $1.0^{-5}$ \\
& & GIN & 128 & 8  & 64 & 0.3 & 1.0 & 4.7  & 0.8 & $1.0^{-5}$ \\

\midrule
{\multirow{6}{*}{Indset}} & {\multirow{3}{*}{Small}} 
& GEN & 512 & 8 & 180 & 0.6 &  1.0 & 4.7 & 2.0  & 0.0 \\
& & GCN & 512 & 8  & 96 & 0.7 & 1.0 & 6.3  & 3.1 & 0.0 \\
& & GIN & 512 & 8  & 180 & 0.7 & 1.0 & 2.4  & 7.5 & $5.6^{-6}$ \\

& {\multirow{3}{*}{Large}} 
& GEN & 128 & 8 & 180 & 0.6 &  1.0 & 4.7 & 2.0  & 0.0 \\
& & GCN & 128 & 8  & 96 & 0.7 & 1.0 & 6.3  & 3.1 & 0.0 \\
& & GIN & 128 & 8  & 180 & 0.7 & 1.0 & 2.4  & 7.5 & $5.6^{-6}$ \\

\midrule
{\multirow{6}{*}{Cauc.}} & {\multirow{3}{*}{Small}} 
& GEN & 512 & 8  & 128 & 0.5 & 1.0 & 6.2  & 6.6 & $1.0^{-7}$ \\
& & GCN & 512 & 8  & 128 & 0.6 & 1.0 & 4.7  & 4.3 & $1.0^{-7}$ \\
& & GIN & 512 & 8  & 128 & 0.4 & 1.0 & 6.2  & 4.1 & $1.2^{-8}$ \\

& {\multirow{3}{*}{Large}} 
& GEN & 128 & 8  & 128 & 0.5 & 1.0 & 6.2  & 6.6 & $1.0^{-7}$ \\
& & GCN & 128 & 8  & 128 & 0.6 & 1.0 & 4.7  & 4.3 & $1.0^{-7}$ \\
& & GIN & 128 & 8  & 128 & 0.4 & 1.0 & 6.2  & 4.1 & $1.2^{-8}$ \\

\midrule
{\multirow{6}{*}{Fac.}} & {\multirow{3}{*}{Small}} 
& GEN & 512 & 8  & 128 & 0.9 & 1.0 & 2.9  & 2.5 & $1.0^{-7}$ \\
& & GCN & 512 & 8  & 96 & 0.7 & 1.0 & 5.3  & 3.8 & 0.0 \\
& & GIN & 512 & 8  & 180 & 0.8 & 1.0 & 1.3  & 4.6 & $1.0^{-7}$ \\

& {\multirow{3}{*}{Large}} 
& GEN & 128 & 8  & 128 & 0.9 & 1.0 & 2.9  & 2.5 & $1.0^{-7}$ \\
& & GCN & 128 & 8  & 96 & 0.7 & 1.0 & 5.3  & 3.8 & 0.0 \\
& & GIN & 128 & 8  & 180 & 0.8 & 1.0 & 1.3  & 4.6 & $1.0^{-7}$ \\

\bottomrule
\end{tabular}
\end{small}
}
\end{center}
\end{table*}

\begin{table*}[t]
\caption{Training hyperparameters of the ODE approach baseline. The experiments are done on the mini-sized instances with GEN-based MPNNs, $v^a$ represents $v \times 10^a$. }
\label{tab:hyperparams_ode}
\begin{center}
\resizebox{\textwidth}{!}{
\begin{small}
\begin{tabular}{ccccccccccc}
\toprule
Instances & MPNN & Candidate & Batch size & Num. layers & Hidden dim. & $\alpha$ & $w_{\text{var}}$ & $w_{\text{obj}}$ & $w_{\text{cons}}$ & Weight decay \\
\midrule
{\multirow{6}{*}{Setcover}} & {\multirow{2}{*}{GEN}}
& Ours & 512 & 3  & 128 & 0.3 & 1.0 & 3.5  & 1.3 & $3.3^{-3}$ \\
& & Baseline & 8 & 3  &  128 & -&- & -  & - & 0.0 \\
\cmidrule{2-11}
& {\multirow{2}{*}{GCN}}
& Ours & 512 & 3  & 180 & 0.8 & 1.0 & 6.1  & 1.6 & $1.0^{-6}$ \\
& & Baseline & 8 & 3  & 180 & - & - & -  & - & 0.0 \\
\cmidrule{2-11}
& {\multirow{2}{*}{GIN}}
& Ours & 512 & 3  & 180 & 0.4 & 1.0 & 3.8  & 1.4 & $5.5^{-6}$ \\
& & Baseline & 8 & 3  & 180 & - & - & -  & - & 0.0 \\

\midrule
{\multirow{6}{*}{Indset}} & {\multirow{2}{*}{GEN}}
& Ours & 512 & 3  & 128 & 0.4 & 1.0 & 7.1  & 6.2 & $1.0^{-6}$ \\
& & Baseline & 8 & 3 & 128 &- & - & -  & - & 0.0 \\
\cmidrule{2-11}
& {\multirow{2}{*}{GCN}}
& Ours & 512 & 3  & 180 & 0.8 & 1.0 & 3.5  & 5.6 & $2.1^{-6}$ \\
& & Baseline & 8 & 3  & 180 & - & - & -  & - & 0.0 \\
\cmidrule{2-11}
& {\multirow{2}{*}{GIN}}
& Ours & 512 & 3  & 180 & 0.4 & 1.0 & 5.9  & 3.9 & $9.7^{-6}$ \\
& & Baseline & 8 & 3  & 180 & - & - & -  & - & 0.0 \\

\midrule
{\multirow{6}{*}{Cauc}}& {\multirow{2}{*}{GEN}}
& Ours & 512 & 3  & 128 & 0.8 & 1.0 & 9.6  & 7.1 & $8.0^{-5}$ \\
& & Baseline & 8 & 3  & 128 & -& - & -  & - & 0.0 \\
\cmidrule{2-11}
& {\multirow{2}{*}{GCN}}
& Ours & 512 & 3  & 180 & 0.7 & 1.0 & 4.5  & 5.0 & $3.7^{-6}$ \\
& & Baseline & 8 & 3  & 180 & - & - & -  & - & 0.0 \\
\cmidrule{2-11}
& {\multirow{2}{*}{GIN}}
& Ours & 512 & 3  & 128 & 0.9 & 1.0 & 6.4  & 5.0 & $1.0^{-6}$ \\
& & Baseline & 8 & 3  & 128 & - & - & -  & - & 0.0 \\

\midrule
{\multirow{6}{*}{Fac}}& {\multirow{2}{*}{GEN}}
& Ours & 512 & 3  & 128 & 0.7 & 1.0 & 5.3 & 0.8 & $1.7^{-6}$ \\
& & Baseline & 8 & 3 & 128 & - & - & -  & - & 0.0 \\
\cmidrule{2-11}
& {\multirow{2}{*}{GCN}}
& Ours & 512 & 3  & 128 & 0.8 & 1.0 & 2.9  & 3.7 & $9.2^{-6}$ \\
& & Baseline & 8 & 3  & 128 & - & - & -  & - & 0.0\\
\cmidrule{2-11}
& {\multirow{2}{*}{GIN}}
& Ours & 512 & 3  & 180 & 0.8 & 1.0 & 1.8  & 1.5 & $3.6^{-7}$ \\
& & Baseline & 8 & 3  & 180 & - & - & -  & - & 0.0 \\

\bottomrule
\end{tabular}
\end{small}
}
\end{center}
\end{table*}

\section{Extended experimental results}

\begin{table*}[t]
\caption{Size generalization. We report the relative objective gap and constraint violation on larger test instances. Numbers represent mean and standard deviation across multiple pretrained models.}
\label{tab:size_gen}
\begin{center}
\resizebox{\textwidth}{!}{
 \begin{small}
\begin{tabular}{ccccccccccc}
\toprule

 & \multicolumn{2}{c}{\textbf{Train size}} & \multicolumn{2}{c}{\textbf{Inference size}} &  \multicolumn{2}{c}{\textbf{GEN}} & \multicolumn{2}{c}{\textbf{GCN}} & \multicolumn{2}{c}{\textbf{GIN}} \\
 
 & Rows & Cols & Rows & Cols &  Obj. (\%) & Cons. &  Obj. (\%) & Cons. &  Obj. (\%) & Cons. \\
 
\midrule
{\multirow{5}{*}{\rotatebox[origin=c]{90}{\scriptsize Setc.}}} & {\multirow{5}{*}{$[300, 500]$}} & {\multirow{5}{*}{$[500, 700]$}}

 & 500 & 700 & 0.717\scriptsize$\pm$0.158 & 0.516\scriptsize$\pm$0.010  & 0.511\scriptsize$\pm$0.047 & 0.509\scriptsize$\pm$0.004  & 1.034\scriptsize$\pm$0.237 & 0.486\scriptsize$\pm$0.023 \\
  & & & 550 & 750 & 0.917\scriptsize$\pm$0.317 & 0.552\scriptsize$\pm$0.012 & 0.871\scriptsize$\pm$0.252 & 0.543\scriptsize$\pm$0.003 & 2.318\scriptsize$\pm$1.411 & 0.497\scriptsize$\pm$0.032 \\
& & & 600 & 700 & 0.993\scriptsize$\pm$0.211 & 0.573\scriptsize$\pm$0.015 & 0.705\scriptsize$\pm$0.125 & 0.565\scriptsize$\pm$0.012 & 1.491\scriptsize$\pm$0.512 & 0.521\scriptsize$\pm$0.045 \\
& & & 500 & 800 & 0.902\scriptsize$\pm$0.323 & 0.528\scriptsize$\pm$0.008 & 1.058\scriptsize$\pm$0.441 & 0.509\scriptsize$\pm$0.004 & 12.538\scriptsize$\pm$16.027 & 0.485\scriptsize$\pm$0.050 \\
 & & & 600 & 800 & 1.004\scriptsize$\pm$0.407 & 0.589\scriptsize$\pm$0.014  & 1.556\scriptsize$\pm$0.588 & 0.568\scriptsize$\pm$0.005  & 12.217\scriptsize$\pm$14.715 & 0.486\scriptsize$\pm$0.071 \\

\midrule

{\multirow{5}{*}{\rotatebox[origin=c]{90}{\scriptsize Indset.}}} & {\multirow{5}{*}{$[584, 990]$}} & {\multirow{5}{*}{$[300, 500]$}}

 & $[978, 994]$ & 500 & 0.128\scriptsize$\pm$0.027 & 0.299\scriptsize$\pm$0.001  & 0.099\scriptsize$\pm$0.008 & 0.303\scriptsize$\pm$0.001  & 0.129\scriptsize$\pm$0.031 & 0.304\scriptsize$\pm$0.001 \\
  & & & $[1028, 1044]$ & 525 & 0.157\scriptsize$\pm$0.063 & 0.300\scriptsize$\pm$0.001  & 0.101\scriptsize$\pm$0.013 & 0.304\scriptsize$\pm$0.001 & 0.111\scriptsize$\pm$0.017 & 0.305\scriptsize$\pm$0.001 \\
  & & & $[1076, 1094]$ & 550 & 0.300\scriptsize$\pm$0.186 & 0.301\scriptsize$\pm$0.002 & 0.096\scriptsize$\pm$0.022 & 0.303\scriptsize$\pm$0.001 & 0.177\scriptsize$\pm$0.097 & 0.304\scriptsize$\pm$0.001\\
  & & & $[1128, 1144]$ & 575 & 1.402\scriptsize$\pm$1.036 & 0.305\scriptsize$\pm$0.006 & 0.146\scriptsize$\pm$0.044 & 0.304\scriptsize$\pm$0.001 & 0.380\scriptsize$\pm$0.367 & 0.304\scriptsize$\pm$0.002 \\
 & & & $[1178, 1194]$ & 600 & 4.552\scriptsize$\pm$3.153 & 0.317\scriptsize$\pm$0.015  & 0.408\scriptsize$\pm$0.317 & 0.304\scriptsize$\pm$0.001  & 0.647\scriptsize$\pm$0.725 & 0.304\scriptsize$\pm$0.002 \\

\midrule

{\multirow{5}{*}{\rotatebox[origin=c]{90}{\scriptsize Cauc.}}} & {\multirow{5}{*}{$[320, 562]$}} & {\multirow{5}{*}{$[300, 499]$}}

 & $[530, 564]$ & 500 & 0.333\scriptsize$\pm$0.134 & 0.257\scriptsize$\pm$0.001  & 0.318\scriptsize$\pm$0.048 & 0.259\scriptsize$\pm$0.001 & 0.344\scriptsize$\pm$0.108 & 0.259\scriptsize$\pm$0.001 \\
 & & & $[596, 646]$ & 500 & 0.363\scriptsize$\pm$0.131 & 0.267\scriptsize$\pm$0.002 & 0.519\scriptsize$\pm$0.069 & 0.270\scriptsize$\pm$0.003 & 0.576\scriptsize$\pm$0.165 & 0.271\scriptsize$\pm$0.001 \\
  & & & $[652, 720]$ & 500 & 0.524\scriptsize$\pm$0.039 & 0.284\scriptsize$\pm$0.001 & 1.255\scriptsize$\pm$0.523 & 0.289\scriptsize$\pm$0.007 & 0.944\scriptsize$\pm$0.114 & 0.289\scriptsize$\pm$0.001 \\
  & & & $[559, 596]$ & 600 & 7.325\scriptsize$\pm$3.615 & 0.257\scriptsize$\pm$0.002 & 0.587\scriptsize$\pm$0.268 & 0.255\scriptsize$\pm$0.001 & 1.014\scriptsize$\pm$0.845 & 0.263\scriptsize$\pm$0.006 \\
 & & & $[633, 677]$ & 600 & 7.965\scriptsize$\pm$3.941 & 0.263\scriptsize$\pm$0.002  & 0.868\scriptsize$\pm$0.441 & 0.258\scriptsize$\pm$0.003  & 1.375\scriptsize$\pm$0.693 & 0.269\scriptsize$\pm$0.005 \\

\midrule

{\multirow{5}{*}{\rotatebox[origin=c]{90}{\scriptsize Fac.}}} & {\multirow{5}{*}{$[441, 900]$}} & {\multirow{5}{*}{$[420, 870]$}}
 & 961 & 930 & 0.912\scriptsize$\pm$0.251 & 0.178\scriptsize$\pm$0.006  & 1.154\scriptsize$\pm$0.206 & 0.173\scriptsize$\pm$0.007  & 1.452\scriptsize$\pm$0.528 & 0.178\scriptsize$\pm$0.003 \\
 & & & 936 & 900 & 1.320\scriptsize$\pm$0.347 & 0.148\scriptsize$\pm$0.009 & 1.615\scriptsize$\pm$0.322 & 0.145\scriptsize$\pm$0.009 & 1.736\scriptsize$\pm$0.558 & 0.153\scriptsize$\pm$0.004 \\
 & & & 936 &  910 & 0.964\scriptsize$\pm$0.063 & 0.209\scriptsize$\pm$0.005 & 1.538\scriptsize$\pm$0.526 & 0.211\scriptsize$\pm$0.007 & 1.538\scriptsize$\pm$0.422 & 0.215\scriptsize$\pm$0.006 \\
 & & & 1116 & 1080 & 1.502\scriptsize$\pm$0.704 & 0.163\scriptsize$\pm$0.009 & 3.540\scriptsize$\pm$3.134 & 0.161\scriptsize$\pm$0.006 & 2.288\scriptsize$\pm$0.659 & 0.167\scriptsize$\pm$0.005 \\
 & & & 1296 & 1260 & 1.808\scriptsize$\pm$0.566 & 0.173\scriptsize$\pm$0.009 & 7.629\scriptsize$\pm$7.577 & 0.179\scriptsize$\pm$0.008 & 13.522\scriptsize$\pm$8.027 & 0.163\scriptsize$\pm$0.021 \\

\bottomrule
\end{tabular}
\end{small}
}
\end{center}
\end{table*}

We provide the training time per epoch and maximal GPU memory usage as supplementary results of \cref{tab: main_table} in \cref{tab:mem_usage}.

\begin{table*}[t]
\caption{Results of our proposed tripartite MPNNs versus bipartite ablations. We report supplementary results of training time per epoch in seconds and the maximal GPU memory allocated in GB.}
\label{tab:mem_usage}
\begin{center}
\resizebox{0.65\textwidth}{!}{
\begin{tabular}{ccccccc|cccc}
\toprule
 & {\multirow{2}{*}{\scriptsize Tri.}} &{\multirow{2}{*}{\scriptsize MPNN}} &  \multicolumn{4}{c}{\textbf{Small instances}} &  \multicolumn{4}{c}{\textbf{Large instances}} \\
 
 & & &  Setcover & Indset & Cauc & Fac &  Setcover & Indset & Cauc & Fac \\

\midrule
{\multirow{6}{*}{\rotatebox[origin=c]{90}{\scriptsize T.(sec)}}}  & {\multirow{3}{*}{\scriptsize \cmark}}
& GEN & 9.057 & 11.652 & 16.190 & 14.910  & 78.193 & 79.035 & 75.225 & 76.166 \\
& & GCN & 6.985 & 7.128 & 7.820 & 7.093  & 26.249 & 36.680 & 31.067 & 21.241 \\
& & GIN & 6.809 & 6.847 & 7.839 & 9.260  & 33.815 & 27.751 & 31.587 & 26.378 \\
\cmidrule{2-11}

& {\multirow{2}{*}{\scriptsize \xmark}}
& GEN & 6.598 & 6.812 & 7.176 & 7.170 & 10.066 & 34.914 & 26.692 & 27.133 \\
& & GCN & 6.560 & 6.665 & 6.899 & 6.742 & 9.694 & 10.788 & 12.267 & 12.285 \\
& & GIN & 6.653 & 6.625 & 6.935 & 6.867 & 9.514 & 13.975 & 11.622 & 21.778 \\

\midrule

{\multirow{6}{*}{\rotatebox[origin=c]{90}{\scriptsize Mem.(GB)}}}   & {\multirow{3}{*}{\scriptsize \cmark}}
& GEN & 23.399 & 32.227 & 51.237 & 48.900  & 67.995 & 69.427 & 65.875 & 71.852\\
& & GCN & 8.643 & 12.669 & 15.597 & 11.297  & 22.064 & 27.086 & 20.288 & 16.246\\
& & GIN & 7.318 & 10.217 & 17.611 & 23.906  & 18.966 & 21.781 & 22.902 & 24.560\\
\cmidrule{2-11}

& {\multirow{2}{*}{\scriptsize \xmark}}
& GEN & 2.462 & 19.098 & 26.031 & 28.096 & 8.104 & 41.103 & 33.250 & 35.623 \\
& & GCN & 1.015 & 4.255 & 8.233 & 7.292 & 2.546 & 9.088 & 9.386 & 10.511 \\
& & GIN & 2.051 & 5.808 & 8.642 & 14.339 & 4.150 & 12.416 & 9.844 & 20.740 \\

\bottomrule
\end{tabular}
}
\end{center}
\end{table*}

\section{IPM solver benchmark}

In this section, we run our Python-based customized IPM solver described in~\cref{alg:ipm-practice} compared with the SciPy official IPM solver on the same set of instances and benchmark the gap between the solvers.

We show the relative and absolute variable value gaps with right to the optimal variable values solved by the SciPy solver, as well as the relative and absolute objective gaps with right to the objective values given by the SciPy solver. For the variable gaps, we present all the variable values of each instance, while the objective gaps are naturally one scalar per instance. Without loss of generality, we pick 100 small set covering instances. The results are displayed in \cref{fig:small_setc}. We can see that the absolute variable gaps are overall satisfyingly small, with the largest value $1.76 \times 10^{-5}$. The relative variable gaps seem larger, as much as $90.28\%$. However, those large gaps correspond to rather small absolute variable values, introducing little numerical influence to the solution. Besides,~\cref{fig:small_setc} shows that the absolute objective gaps are at $1 \times 10 ^{-5}$ level, while the relative gaps are at most $1.61 \times 10 ^{-4} \%$, or even smaller. 

\begin{figure}[t]
  \begin{center}
      \centering
      \includegraphics[width=0.9\linewidth]{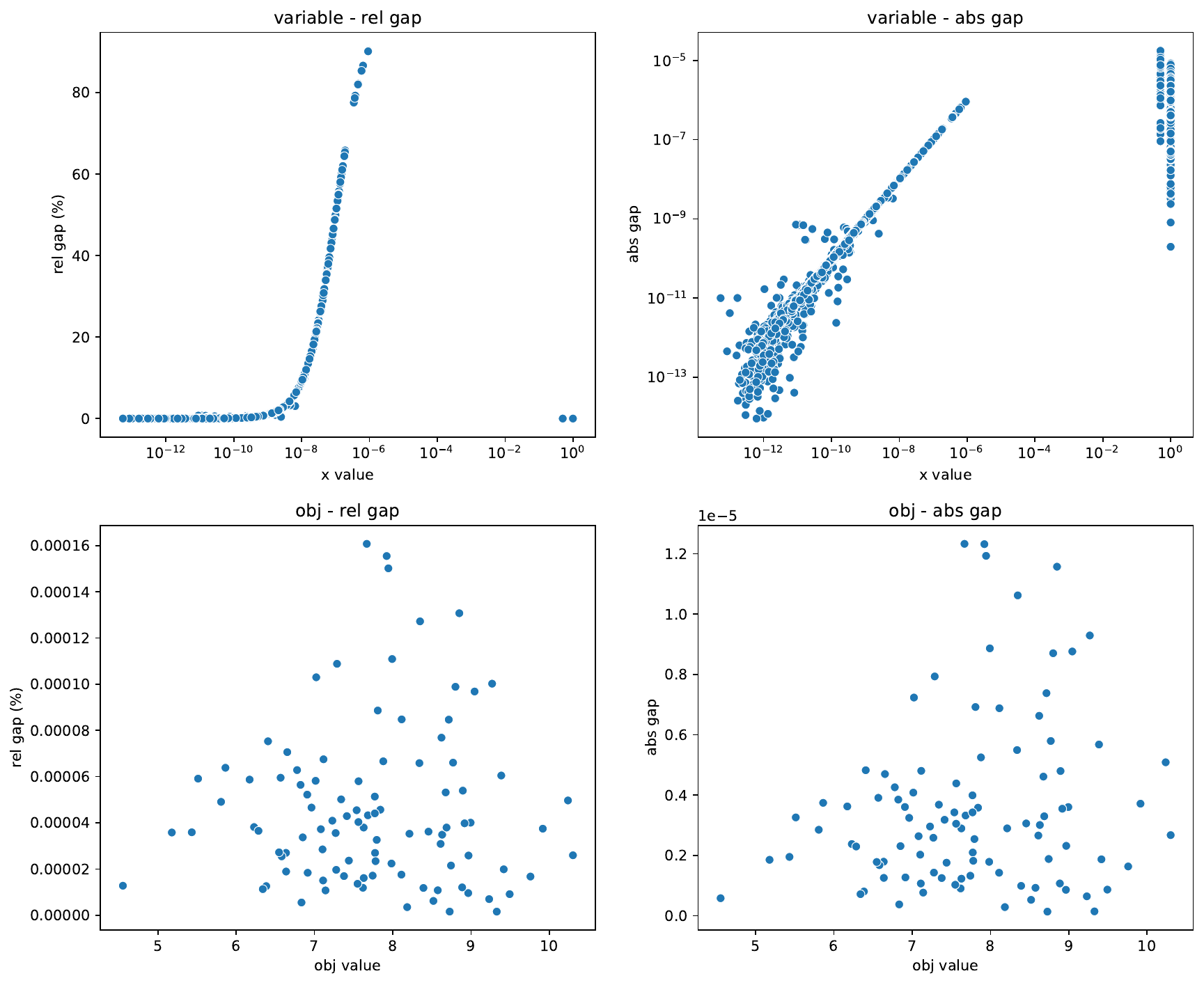}
      \caption{Solver benchmarking on small set covering relaxation instances.}
      \label{fig:small_setc}
  \end{center}
\end{figure}

Let us take a look at another case of small maximal independent set instances. In \cref{fig:small_indset}, the variable gaps at large ground truth variable values are surprisingly high, with a 0.19 absolute value gap and $55.29\%$ relative gap. However, the objective gaps remain low. In light of verifying the correctness of the constraints and the variable range, we conclude that our solver converges to alternative optimal solutions when multiple solutions exist.

\begin{figure}[t]
  \begin{center}
      \centering
      \includegraphics[width=0.9\linewidth]{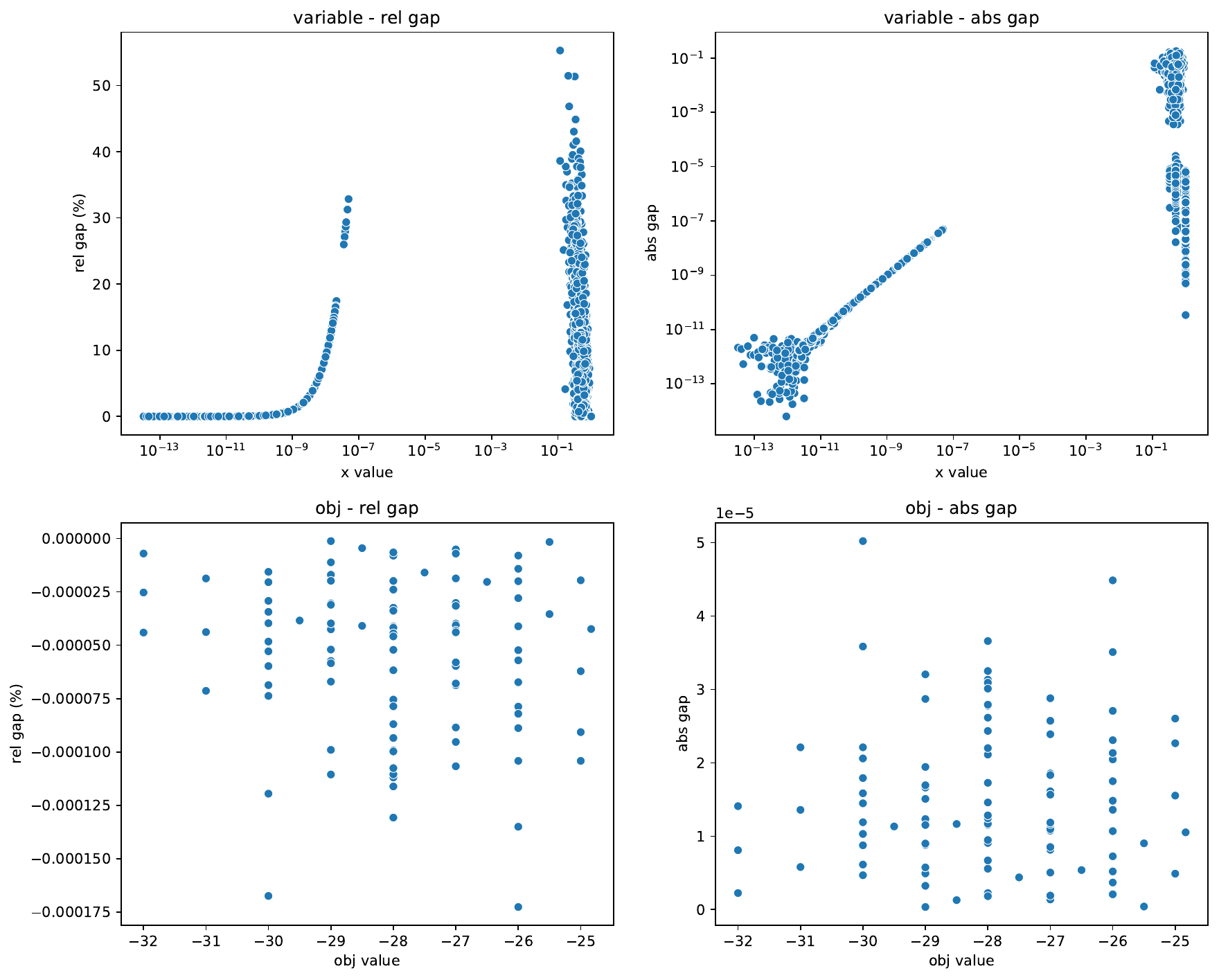}
      \caption{Solver benchmarking on small maximal independent set relaxation instances.}
      \label{fig:small_indset}
  \end{center}
\end{figure}
\end{document}